\documentclass[10pt,journal,compsoc]{IEEEtran}
\usepackage{amssymb}
\usepackage{amsmath}
\usepackage{amsfonts}
\usepackage{amsmath}
\usepackage{mathabx}
\usepackage{amssymb}
\usepackage{amsopn}
\usepackage{bm}
\usepackage{relsize}
\usepackage{exscale}
\usepackage{booktabs}
\usepackage{graphicx}
\usepackage{epstopdf}
\usepackage{epsfig}
\usepackage{ragged2e}
\usepackage{url}
\usepackage{float}
\usepackage{stfloats}
\usepackage{subfigure}
\usepackage{algorithmic}
\usepackage{algorithm}
\usepackage{threeparttable}
\usepackage{lineno}
\usepackage{multirow}
\usepackage{amsthm}
\newtheorem{definition}{Definition}
\newtheorem{theorem}{Theorem}
\newtheorem{lemma}{Lemma}
\newtheorem{corollary}{Corollary}
\usepackage{color}

\ifCLASSOPTIONcompsoc
  \usepackage[nocompress]{cite}
\else
  \usepackage{cite}
\fi
\ifCLASSINFOpdf
\else
\fi
\begin{document}

\title{Unlock the Power of Algorithm Features: A Generalization Analysis for Algorithm Selection}

\author{Xingyu~Wu,
        Yan~Zhong,
        Jibin~Wu,
        Yuxiao~Huang,
        Sheng-hao~Wu,
        and~Kay Chen~Tan,~\IEEEmembership{Fellow,~IEEE}
\IEEEcompsocitemizethanks{\IEEEcompsocthanksitem X. Wu, J. Wu, Y. Huang, S. Wu, and KC. Tan are with the Department of Computing, The Hong Kong Polytechnic University, Hong Kong SAR 999077, China (e-mail: xingy.wu@polyu.edu.hk; jibin.wu@polyu.edu.hk, yuxiao.huang@polyu.edu.hk, sheng-hao.wu@polyu.edu.hk, kaychen.tan@polyu.edu.hk). Y. Zhong is with the School of Mathematical Sciences, Peking University, Beijing 100871, China (e-mail: zhongyan@stu.pku.edu.cn).
}
}

\IEEEtitleabstractindextext{%
\begin{abstract}
In the algorithm selection research, the discussion surrounding algorithm features has been significantly overshadowed by the emphasis on problem features. Although a few empirical studies have yielded evidence regarding the effectiveness of algorithm features, the potential benefits of incorporating algorithm features into algorithm selection models and their suitability for different scenarios remain unclear. In this paper, we address this gap by proposing the first provable guarantee for algorithm selection based on algorithm features, taking a generalization perspective. We analyze the benefits and costs associated with algorithm features and investigate how the generalization error is affected by different factors. Specifically, we examine adaptive and predefined algorithm features under transductive and inductive learning paradigms, respectively, and derive upper bounds for the generalization error based on their model's Rademacher complexity. Our theoretical findings not only provide tight upper bounds, but also offer analytical insights into the impact of various factors, such as the training scale of problem instances and candidate algorithms, model parameters, feature values, and distributional differences between the training and test data. Notably, we demonstrate how models will benefit from algorithm features in complex scenarios involving many algorithms, and proves the positive correlation between generalization error bound and $\chi^2$-divergence of distributions.

\end{abstract}

}

\maketitle

\IEEEdisplaynontitleabstractindextext

\IEEEpeerreviewmaketitle

\section{Introduction}
\label{introduction}

For most computational problems, it is widely recognized that no single algorithm outperforms all others on every problem instance \cite{kerschke2019automated}. This phenomenon, known as performance complementarity, has been observed across various domains, including machine learning tasks \cite{munoz2018instance,raschka2018model}, optimization problems \cite{prager2022automated,munoz2015algorithm}, NP-hard problems \cite{heins2023study,tierney2015algorithm}, and so on. To address this challenge, per-instance algorithm selection has been extensively studied, which aims to determine the algorithm that is expected to perform best on a specific problem instance \cite{brazdil2018metalearning}. Existing algorithm selection techniques can be categorized into two groups based on their consideration of algorithm features: approaches based on problem features and approaches based on both problem and algorithm features (referred to as problem feature-based methods and algorithm feature-based methods hereafter for brevity). Fig. \ref{Fig_Introduction} illustrates the distinctions between these two methodologies. 

Many problem feature-based methods are primarily grounded in the framework proposed by Rice \cite{rice1976algorithm}, which conceptualizes algorithm selection as a mapping from problem features to algorithms, without taking into account algorithm features. These methods often formulate the algorithm selection task as either a performance regression or a multi-class classification problem. Specifically, regression-based techniques \cite{xu2008satzilla} directly predict the performance of each algorithm based on the problem feature set. Conversely, multi-class classification approaches \cite{cunha2018label,abdulrahman2018speeding,hu2021cascaded} assign scores to algorithms based on their relative suitability for a specific problem and select the most appropriate algorithm based on these scores. In addition, several other methodologies have been proposed. For instance, some studies formalize algorithm selection as a collaborative filtering problem \cite{misir2017alors,fusi2018probabilistic} and utilize a sparse matrix with performance data available for only a few algorithms on each problem instance. And similarity-based methods \cite{amadini2014sunny, kadioglu2010isac} select algorithms based on the similarity between the current problem instance and previously encountered instances. There are also some hybrid methods \cite{hanselle2020hybrid,fehring2022harris} combine multiple techniques above.

\begin{figure*}[t]
\begin{center}
\includegraphics[width=0.85\textwidth]{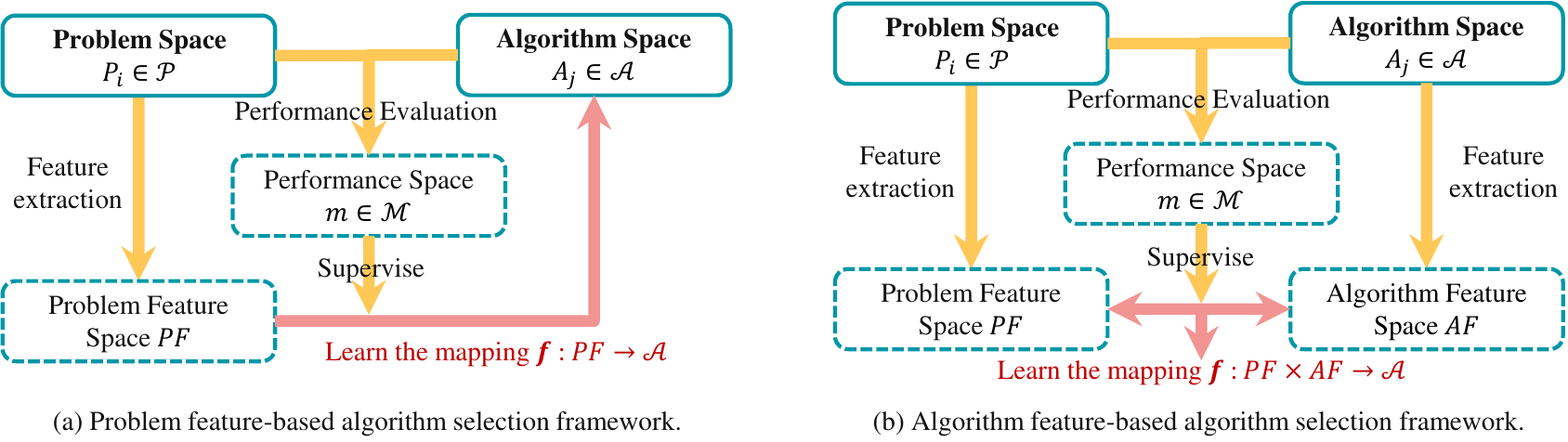}
\end{center}
\caption{Comparison of the problem feature-based framework and the algorithm feature-based framework.}
\label{Fig_Introduction}
\end{figure*}

The majority of research in algorithm selection has primarily focused on problem features, with only a limited number of studies \cite{hough2006modern,tornede2020extreme,hilario2009data,pulatov2022opening,de2021explorative,wu2023llm} exploring and leveraging algorithm features. In these studies, algorithm features are primarily represented in two forms. Most of them utilize manually extracted parameters in algorithm or automatically derived features from code to characterize the unique aspects of algorithms, such as algorithm hyperparameters \cite{hough2006modern,tornede2020extreme}, model structure information \cite{hilario2009data}, code-related statistical features \cite{pulatov2022opening}, abstract syntax tree features \cite{pulatov2022opening}, and large language model-extracted code features \cite{wu2023llm}. These features, collectively referred to as predefined features, describe the logical and computational characteristics of the algorithms themselves. The extraction process of these features is independent of specific problem instances and can be directly observed and acquired before model training. Another type of algorithm feature is obtained through embedding layers in deep models \cite{wu2023llm,ruhkopf2022masif}. During the model training process, embedding layers can automatically learn a set of compact feature vectors that represent algorithmic information, by learning the similarities between algorithms and the correlations between candidate algorithms and problem instances. These features, known as adaptive features, are directly relevant to the algorithm selection task. Algorithm feature-based models typically concatenate algorithm features with problem features and employ machine learning techniques to either regress algorithm performance, predict the most suitable algorithm, or calculate the matching degree between the problem-algorithm pair, to make compatibility decisions.

While these studies have provided some evidence of the positive impact of algorithm features on the algorithm selection task, they are still insufficient to offer compelling guidance for the application of algorithm features in real-world solutions. On one hand, within the limited scope of empirical research on algorithm features, the comparison between models incorporating algorithm features and models utilizing only problem features has yielded mixed results in different scenarios, as evidenced by studies such as \cite{pulatov2022opening,wu2023llm}. The conditions under which algorithm features should be employed, as well as the specific types of algorithm features to consider, remain unclear. On the other hand, several factors, including the scale of training data and the extent of distributional shifts, can significantly influence the design of algorithm selection approaches. The rigorous analysis of how to incorporate these factors into the decision-making process of algorithm feature utilization is currently lacking. Consequently, the use of algorithm features remains primarily at the stage of experimental validation, with ongoing debates regarding their effectiveness and appropriate application conditions for practical deployment. 

Intuitively, considering algorithm features in the context of algorithm selection may enhance the predictive capacity of models by introducing additional information. If these features prove informative and there is sufficient training data, the model hold the potential for achieving better performance on testing data \cite{wu2023llm}. However, larger-capacity models necessitate a larger quantity of training data, which may pose challenges in scenarios where collecting performance data for candidate algorithms is prohibitively expensive. Notably, algorithm features appear to improve the model's ability to generalize to changed distribution, encompassing not only diverse problem instances but also novel algorithms that were not encountered during training. The degree to which the model benefits from algorithm features in terms of inductive generalization is contingent upon various factors, including the model's theoretical complexity, the scale of the problem instances and candidate algorithms involved, the property of networks and loss functions, as well as the distribution variation, among others. This paper aims to explore the generalization nature of algorithm feature-based models and provide a theoretical analysis of how the interplay of these factors influences the model's generalization performance. By this means, we aim to elucidate the advantage and trade-off associated with algorithm features.

Before the incorporation of algorithm features, existing methods are predominantly designed within the transductive learning paradigm \cite{el2009transductive}, with little deliberate consideration given to the theoretical generalization error of these approaches. While not explicitly stated in these studies, most of them assumed the training and test sets shared the same distribution\footnote{This statement is limited to the algorithm selection on static scenarios only. It is worth noting that some existing algorithm selection methods for streaming data \cite{van2014algorithm,van2018online} may take into account distribution changes in the problem stream \cite{rossi2014metastream}. However, the scope of discussion in this paper is limited to static data.} and trained models solely on problem features. Once algorithm features are considered, both adaptive features and predefined features will impart the model with diverse aspects of generalization capabilities across different learning paradigms. Adaptive features are derived based on the adaptability between algorithms and problems within the training data, making them unable to generalize to new candidate algorithms. Conversely, predefined features directly capture the intrinsic properties of algorithms, enabling them to exhibit generalization performance even in scenarios with distribution shifts. Therefore, we will explore the model generalization under two paradigms, i.e., transductive \cite{el2009transductive} and inductive \cite{koltchinskii2001rademacher} learning settings. To ensure the broad applicability of our theoretical research, we adopt a higher level of abstraction for modeling, i.e., the multilayer neural network with arbitrary depths and hidden layer designs, employing 1-Lipschitz activation functions. Under both learning paradigms, we provide the upper bounds of the Rademacher complexity and generalization error for models based on adaptive algorithm features, predefined algorithm features, as well as regression and multi-class classification models based on problem features in Sections 3 and 4. Based on these theoretical investigations, we draw the following conclusions to inspire the practical application of algorithm features in real-world scenarios:
\begin{itemize}
  \item The impact of training set size on the Rademacher complexity of models differs between the transductive learning and inductive learning paradigms. As the size of training set (denoted as $|\mathcal{S}|$) increases, the Rademacher complexity of transductive learning models and inductive learning models decreases at a rate of $\frac{1}{|\mathcal{S}|^{\frac{1}{2}}}$ and $\frac{1}{|\mathcal{S}|^{\frac{1}{4}}}$, respectively. In scenarios where distribution differences are considered, the use of predefined features requires a larger training set.
  \item In the context of transductive learning, algorithm features can enhance the generalization of models in scenarios with a relatively large number of candidate algorithms. When the size of training set is small, as the number of problem instances (denoted as $|\mathcal{S}_{\mathcal{P}}|$) and the number of the candidate algorithms (denoted as $|\mathcal{S}_{\mathcal{A}}|$) increase in the training phase, the upper bound of the generalization error for algorithm feature-based models decreases at a rate of $\frac{1}{|\mathcal{S}_{\mathcal{A}}| \cdot |\mathcal{S}_{\mathcal{P}}|}$. In the same setting, the performance regression models and multi-class classification models constructed based on problem features decrease at rates of $\frac{1}{|\mathcal{S}_{\mathcal{P}}|}$ and $\frac{|\mathcal{S}_{\mathcal{A}}|}{|\mathcal{S}_{\mathcal{P}}|}$, respectively. However, with abundant training data, the upper bound of the generalization error for algorithm feature-based models decreases at a rate of $\frac{1}{\sqrt{|\mathcal{S}_{\mathcal{A}}| \cdot |\mathcal{S}_{\mathcal{P}}|}}$, demonstrating superior performance compared to the performance regression models and multi-class classification models constructed based on problem features, which decrease at rates of $\frac{1}{\sqrt{|\mathcal{S}_{\mathcal{P}}|}}$. Additionally, the impact of the number of problems on different models exhibits similarity.
  \item When the distribution shift exists between the training and test sets, models based on adaptive features lack generalization ability. Conversely, predefined algorithm feature-based models possesses generalization and experience an increase in generalization error as the degree of distribution discrepancy grows, with a growth rate proportional to the $\frac{3}{4}$ power of the chi-square distance between the training and test sets. Furthermore, the size and parameters of each layer in the model amplify the generalization error, with a scaling factor equal to the Frobenius norm of the parameter matrix. Consequently, in scenarios with significant distribution shifts, it is advisable to avoid using large-scale models.
\end{itemize}

\section{Algorithm Feature-based Models and Learning Paradigms}

This section commences by providing an overview of the basic notation and presenting detailed definitions for models based on algorithm features. It subsequently delves into the mechanisms of transductive and inductive learning, elucidating the respective definitions of Rademacher Complexity for models operating within each learning paradigm.

\subsection{Model Definition and Basic Notations}

To discuss the impact of algorithm features, we first define the algorithm selection model based on different types of algorithm features. In this paper, we discuss two types of algorithm features, corresponding to different learning mechanism, namely adaptive features and predefined features. 
Predefined features, derived from a automatic or manual extraction process, can capture inherent algorithmic properties, e.g., features from algorithmic code or algorithm description. On the other hand, adaptive features learned through matching relationships between algorithms and problem instances (achieved by embedding layer), provide nuanced insights into algorithmic suitability for specific problem domains.


Intuitively, the roles of adaptive features and predefined features in algorithm selection tasks are different. Predefined feature is a type of explicit representation extracted from algorithm-related knowledge, representing the intrinsic characteristics of the algorithm. 
Adaptive features, on the other hand, are a type of implicit representation that does not explicitly represent a specific intrinsic characteristic of the algorithm. Instead, they are a set of features automatically defined based on the relationship between the problem and the algorithm. These features are designed specifically for the candidate algorithms in the training set, and feature extraction is performed synchronously during the model training process.

Correspondingly, the learning paradigms of the constructed learning models based on adaptive features and predefined features are also different. The model based on adaptive features learns to extract and utilize these features to improve its performance on the selection of candidate algorithms encountered in training data. Due to the compatibility between features and the algorithm-problem relationship, adaptive features usually lack generalization capacity over algorithms. In other words, they are only effective for algorithms that have appeared in the training set. However, since all features are relevant to the specific algorithm selection scenario, they can perform well on these candidate algorithms. Conversely, predefined features have the advantage of generalization over algorithms, which are independent of the algorithm-problem relationship. The predefined features are used as an initial representation, and then the model is further trained on the specific task using labeled data. During this process, the model adjusts its parameters to better fit the task-specific data, leveraging the knowledge captured in the predefined features. Therefore, given the same scale of training data, models based on predefined features may be more challenging to train, but they also exhibit good generalization on algorithms that are not present in the candidate set.

We provide the notations in the following. Given a dataset $\mathcal{D}$, consisting of a problem set $\mathcal{P}$, a algorithm set $\mathcal{A}$ for solving problem instances in $\mathcal{P}$, and a performance metric $\mathcal{M}$: $\mathcal{P}\times\mathcal{A} \rightarrow \mathbb{R}$ which quantifies the performance of any algorithm $A_j\in\mathcal{A}$ on each problem instance $P_i\in\mathcal{P}$, algorithm feature-based algorithm selection model should construct a selector $\mathbf{f}: \mathcal{P}\times\mathcal{A} \rightarrow \mathbb{R}$ or $\mathbf{f}: \mathcal{P}\times\mathcal{A} \rightarrow \{0,1\}$ that determines the compatibility (or degree of compatibility) between any algorithm and any problem instance. The dataset $\mathcal{D}$ is divided into a training set and a test set, represented as $\mathcal{S}$ and $\mathcal{T}$, respectively. Their sizes are denoted as $|\mathcal{S}|$ and $|\mathcal{T}|$. In methods that do not consider algorithm features, most studies require learning the mapping from $\mathcal{P}$ to $\mathcal{A}$, such as algorithm selection approaches based on multi-class classification methods or performance regression methods. In this case, the size of the training data is equal to the number of problem instances. When the role of algorithm features is fully considered, it is necessary to define the sets of both problems and algorithms for the training and test purpose, denoted as $\mathcal{S}=\{\mathcal{S}_{\mathcal{P}}\cup\mathcal{S}_{\mathcal{A}}\}$ and $\mathcal{T}=\{\mathcal{T}_{\mathcal{P}}\cup\mathcal{T}_{\mathcal{A}}\}$, respectively. For the rest of this paper, unless otherwise specified, the candidate algorithms in $\mathcal{S}_{\mathcal{A}}$ and $\mathcal{T}_{\mathcal{A}}$ are assumed to be the same, i.e., $\mathcal{S}_{\mathcal{A}} = \mathcal{T}_{\mathcal{A}}$. Moreover, since the model's input is a problem-algorithm pair, the size of the training data and test data corresponds to $|\mathcal{S}|=|\mathcal{S}_{\mathcal{P}}| \cdot |\mathcal{S}_{\mathcal{A}}|$ and $|\mathcal{T}|=|\mathcal{T}_{\mathcal{P}}| \cdot |\mathcal{T}_{\mathcal{A}}|$, respectively.

In the following theoretical analysis, we will discuss a simple yet sufficiently universal and effective model. The input will be fed into a deep neural network $\mathbf{f}$ with $l$ layers, and the parameters in each layer are denoted as $W^{(i)} (i\in\{1,2,\cdots,l\})$.  Any 1-Lipschitz activation functions (e.g., ReLU, Leaky ReLU, SoftPlus, Tanh, Sigmoid, ArcTan, or Softsign) can be used. For model based on adaptive features, the input includes the problem features and the algorithm index. For ease of modeling, the algorithm features in the embedding layer are combined with problem features, which are collectively represented as the parameters of the first layer of the deep network, denoted as $W^{(0)} = \begin{bmatrix} PF_{1}, PF_{1}, \cdots \\ AF_{1}, AF_{2}, \cdots \end{bmatrix}$, where $PF_{i}$ and $AF_{j}$ are frozen problem features and optimized algorithm features. Correspondingly, the input is a one-hot column vector, where non-zero positions signify the indices of the problem and algorithm. The main body of the model remains consistent with the previous definition. We denote the model based on adaptive features as Model$_{a}$. For model based on predefined features, the input includes the problem features and the predefined algorithm features. The model body is the same as the previously defined model. We denote the model based on predefined features as Model$_{b}$.

\subsection{Transductive and Inductive Generalization in Algorithm Selection}

In the context of algorithm selection, the choice of an appropriate algorithm for each problem can be understood in transductive manner or inductive manner. The distinction between transductive learning and inductive learning in this scenario lies in whether the underlying distribution of the training and testing data are the same or different. In the cases where the underlying distribution of both the training and testing datasets align, the approach can be considered a form of transductive learning. This implies that the learning model is expected to generalize its patterns to specific testing data based on the training data. On the other hand, when the underlying distribution of the training and testing datasets differ, the scenario can be categorized as inductive learning. In this case, the model aims to generalize patterns to samples drawn from some unknown distribution, representing a more generalized learning setting.

The rationale behind this distinction can be elucidated by the nature of the data and the underlying assumptions. Transductive learning, in the algorithm selection context, is akin to anticipating that the testing data will exhibit characteristics similar to the training data, as both involve the same set of candidate algorithms. This aligns with the idea of predicting well on a specific subset of entries whose characteristics are known in advance. Conversely, inductive learning in algorithm selection assumes a more versatile approach. It acknowledges that the candidate algorithms in testing samples may differ from the training set, reflecting a broader range of potential scenarios, especially in large-scale problems. Therefore, we will analyze the transductive generalization performance of Model$_{a}$ and the transductive/inductive generalization performance of Model$_{b}$. This is because the embedding layer cannot represent algorithms that are not present in the training set, while predefined features are applicable to unknown candidate algorithms.

To evaluate the ability of a model to make predictions on new and unseen data, transductive generalization and inductive generalization analyses are both necessary to reflect the different aspects of generalization ability. Transductive generalization refers to the process of generalizing from specific observed instances to make predictions about other specific instances within the same dataset. While inductive generalization involves generalizing from a set of observed instances to make predictions about unseen instances that may come from a different but related distribution. To understand the factors that influence generalization, two types of Rademacher complexities are introduced, which provide a way to analyze the complexity of a hypothesis class and its impact on generalization. The definitions of transductive Rademacher complexity \cite{el2009transductive} and inductive Rademacher complexity \cite{koltchinskii2001rademacher} are introduced as follows:

\begin{definition}
\label{tran_R_comp}
(Transductive Rademacher Complexity) Let $\mathcal{V}\subset\mathbb{R}^{m+u}$ and $p\in \left[ 0, \frac{1}{2} \right]$. Let $\bm{\sigma} = \left( \sigma_1,\dots,\sigma_{m+u} \right)^T$
be a vector of $i.i.d.$ random variables such that
\begin{equation}
\sigma_i \triangleq\left\{\begin{array}{lll}
1 & \text { with probability } & p ; \\
-1 & \text { with probability } & p ; \\
0 & \text { with probability } & 1-2 p
\end{array}\right.
\end{equation}
The transductive Rademacher complexity with parameter $p$ is:
\begin{equation}
\mathfrak{R}_{m+u}(\mathcal{V}, p) \triangleq\left(\frac{1}{m}+\frac{1}{u}\right) \cdot \mathbf{E}_{\boldsymbol{\sigma}}\left\{\sup _{\mathbf{v} \in \mathcal{V}} \boldsymbol{\sigma}^T \cdot \mathbf{v}\right\}
\end{equation}
where $\mathbf{E}$ denotes the mathematical expectation in this paper.
\end{definition}

\begin{definition}
\label{in_R_comp}
(Inductive Rademacher Complexity) Let $\mathcal{D}$ be a probability distribution over $\mathcal{X}$ . Suppose that the examples $X_n = \{ x_i \}_{i=1}^n$ are sampled independently from $\mathcal{X}$ according to $\mathcal{D}$. Let $\mathcal{V}$ be a class of functions mapping $\mathcal{X}$ to $\mathbb{R}$. Let $\bm{\sigma} = \{ \sigma_i \}_{i=1}^n$ be an independent uniform $\{\pm 1\}$-valued random variables, $\sigma_i=1$ with probability $\frac{1}{2}$ and $\sigma_i=-1$ with the same probability. The empirical Rademacher complexity is:
\begin{equation}
\widehat{\mathfrak{R}}_n(\mathcal{V}) \triangleq \frac{1}{n} \mathbf{E}_{\boldsymbol{\sigma}}\left\{\sup _{\mathbf{v} \in \mathcal{V}} \sum_{i=1}^n \sigma_i v\left(x_i\right)\right\}
\end{equation}
and the Rademacher complexity of $\mathcal{F}$ is
\begin{equation}
\mathfrak{R}_n(\mathcal{V}) \triangleq \mathbf{E}_{X_n \sim \mathcal{D}^n}\left\{\widehat{\mathfrak{R}}_n^{(\text {ind })}(\mathcal{V})\right\}.
\end{equation}
\end{definition}

Let $\mathcal{S}$ and $\mathcal{T}$ denote training set and test set. Transductive Rademacher complexity is usually used to bound the test error:
\begin{equation}
\mathcal{L}_\mathcal{T}(\mathbf{f})=\frac{1}{|\mathcal{T}|} \sum_{x_i,y_i\in\mathcal{T}} \ell\left(\mathbf{f}(x_i), y_i\right)
\end{equation}
where $\mathbf{f}$ denotes the model. The aim is different from the full sample error bounded by the inductive Rademacher complexity, formulated as:
\begin{equation}
\mathcal{L}_{\mathcal{S}+\mathcal{T}}(\mathbf{f})=\frac{1}{|\mathcal{S}|+|\mathcal{T}|} \sum_{x_i,y_i\in\mathcal{S}\cup\mathcal{T}} \ell\left(\mathbf{f}(x_i), y_i\right).
\end{equation}

In the following, we will analyze the upper bound of the generalization error based on these definitions.

\section{Generalization Error of the Adaptive Feature-based Model}

Compared to using only problem features, incorporating both problem and algorithm features usually increases the model's capacity, capturing more comprehensive information. This section analyzes the theoretical performance based on Model$_{a}$. 
According to Definition \ref{tran_R_comp}, the transductive Rademacher Complexity can be calculated as:
\begin{equation}
\begin{aligned}
& \mathfrak{R}_{|\mathcal{S}|+|\mathcal{T}|}(\mathcal{L}, p) = \\ & \left(\frac{1}{|\mathcal{S}|}+\frac{1}{|\mathcal{T}|}\right) \mathbb{E}_{\boldsymbol{\sigma}}\left[ \sup \sum_{x_i,y_i\in\mathcal{S}} \sigma_i \ell\left(\mathbf{f}(x_i), y_i\right) \right]
\end{aligned}
\end{equation}
where $\sigma_i$ is the Rademacher random variable of the $i$-th instance, and $\mathcal{L}$ is a class of loss functions.

In the following, we first introduce a variation of the widely recognized `contraction principle' in the field of Rademacher averages theory, and then solve the upper bound of $\mathfrak{R}_{|\mathcal{S}|+|\mathcal{T}|}(\mathcal{L})$.

\begin{lemma}
\label{theorem2}
(Contraction of Rademacher Complexity, following from \cite{el2009transductive}) Let $\mathcal{V}\in \mathbb{R}^{m+u}$ be a set of vectors. Let $f$ and $g$ be real-valued functions. Let $\boldsymbol{\sigma}=\{\sigma_i\}_{i=1}^{m+u}$ be Rademacher variables as defined in Definition \ref{tran_R_comp}. If for all $1\leq i \leq m+u$ and any $\mathbf{v}, \mathbf{v}^{\prime} \in \mathcal{V}$, $|f(v_i)-f(v_i^{\prime})| \leq |g(v_i)-g(v_i^{\prime})|$, then
\begin{equation}
\mathbf{E}_{\boldsymbol{\sigma}} \sup _{\mathbf{v} \in \mathcal{V}}\left[\sum_{i=1}^{m+u} \sigma_i f\left(v_i\right)\right] \leq \mathbf{E}_{\boldsymbol{\sigma}} \sup _{\mathbf{v} \in \mathcal{V}}\left[\sum_{i=1}^{m+u} \sigma_i g\left(v_i\right)\right]
\end{equation}
\end{lemma}

In deriving the upper bounds for transductive Rademacher complexity in this paper, Lemma \ref{theorem2} will assist us in seeking more compact theoretical solutions above deep neural network models, as shown in Theorem \ref{bound1}.

\begin{theorem}
\label{bound1}
Let $|\mathcal{S}|$ and $|\mathcal{T}|$ denote the scale of training and test samples, $\mathcal{L}$ is a class of the loss function with Lipschitz constant $\gamma$, $L$ is the Lipschitz constant of the model excluding algorithm embedding layer, and $W^{(0)}$ is the parameters composed of the problems features and parameters in algorithm embedding layer. Then the transductive Rademacher complexity of Model$_{a}$, $\mathfrak{R}_{|\mathcal{S}|+|\mathcal{T}|}(\mathcal{L}, p)$, is bounded by:
\begin{equation}
\begin{aligned}
\mathfrak{R}_{|\mathcal{S}|+|\mathcal{T}|}(\mathcal{L}, p) \leq \frac{\gamma L\|W^{(0)}\|_2(|\mathcal{S}|+|\mathcal{T}|)}{|\mathcal{S}||\mathcal{T}|}
\end{aligned}
\end{equation}
\end{theorem}


\begin{proof}
According to the Lemma 5 in literature \cite{el2009transductive}, since $\ell\left(\mathbf{f}(x_i), y_i\right)$ satisfies the Lipschitz condition: $| \ell\left(\mathbf{f}(x_i), y_i\right) - \ell\left(\mathbf{f}(x_j), y_j\right) | \leq \gamma |\mathbf{f}(x_i)-\mathbf{f}(x_j)|$ \cite{neyshabur2017exploring}, we have
\begin{equation}
\begin{aligned}
& \mathfrak{R}_{|\mathcal{S}|+|\mathcal{T}|}(\mathcal{L}, p) \leq \\ & \left(\frac{1}{|\mathcal{S}|}+\frac{1}{|\mathcal{T}|}\right) \mathbf{E}_{\boldsymbol{\sigma}}\left[ \sup \sum_{x_i\in\mathcal{S}} \gamma\sigma_i \mathbf{f}(x_i) \right] = \gamma\mathfrak{R}_{|\mathcal{S}|+|\mathcal{T}|}(\mathcal{F}, p)
\end{aligned}
\end{equation}
where $\mathcal{F}$ is the space of variable $\mathbf{f}(x_i)$.

According to \cite{virmaux2018lipschitz}, the Multi-Layer Perceptron with a non-linear activation function is Lipschitz continuous with a constant such that
\begin{equation}
\forall x_i, x_j \in \mathbb{R}^n,\|\mathbf{f}(\theta_1(x_i))-\mathbf{f}(\theta_0(x_j))\|_2 \leq L\|\theta_1(x_i)-\theta_0(x_j)\|_2
\end{equation}
where $L$ is the Lipschitz constant of sub-model from layer 1 to layer $l$, and $\theta_0$ represent the output of the first layer. $L$ is bounded by:
\begin{equation}
\label{eq_Lipschitz}
\begin{aligned}
L=\sup _{x_i \in \mathbb{R}^n} \|W^{(l)} \operatorname{diag}\left(g_{l-1}^{\prime}\left(\theta_{l-1}\right)\right) W^{(l-1)} \\ \ldots W^{(2)} \operatorname{diag}\left(g_1^{\prime}\left(\theta_1\right)\right) W^{(1)} \|_2
\end{aligned}
\end{equation}
where $g_{i}^{\prime}$ and $\theta_i$ represent the derivative and output of the $i$-th layer. More efficient and accurate estimation of the value of $L$ can be found in \cite{virmaux2018lipschitz,fazlyab2019efficient,shi2022efficiently}.

The same can be obtain:
\begin{equation}
\forall x_i, x_j \in \mathbb{R}^n, \|\theta_1(x_i)-\theta_1(x_j)\|_2 \leq \|W^{(0)}\|_2 \|x_i-x_j\|_2
\end{equation}
where $\|W^{(0)}\|_2$ is the Lipschitz constant of the first layer, where linear operation is performed and the inputs are one-hot vectors. Therefore, the maximum singular value of the weight matrix $W^{(0)}$ is the Lipschitz constant.

Based on Eq. (\ref{eq_Lipschitz}), we construct a function $\mathbf{g}: \mathbb{R}^n\rightarrow \mathbb{R}$ to use Lemma \ref{theorem2} to bound the $\mathfrak{R}_{|\mathcal{S}|+|\mathcal{T}|}(\mathcal{F}, p)$. Let $v$ is a constant vector satisfying $\|v\|_2\leq 1$ and $\mathbf{g}(x) = \left\langle x,v\right\rangle$, then according to Cauchy-Schwarz inequality, $\forall x_i, x_j \in \mathbb{R}^n$, we have
\begin{equation}
\begin{aligned}
|\mathbf{f}(x_i)-\mathbf{f}(x_j)| & \leq L\|W^{(0)}\|_2\|x_i-x_j\|_2 \\ & \leq L\|W^{(0)}\|_2 |\mathbf{g}(x_i)-\mathbf{g}(x_j)|
\end{aligned}
\end{equation}
According to the Lemma \ref{theorem2}, we have
\begin{equation}
\label{eq_bound}
\begin{aligned}
& \mathfrak{R}_{|\mathcal{S}|+|\mathcal{T}|}(\mathcal{F}, p) \\ \leq & L\|W^{(0)}\|_2 \left(\frac{1}{|\mathcal{S}|}+\frac{1}{|\mathcal{T}|}\right) \mathbb{E}_{\boldsymbol{\sigma}}\left[ \sup \sum_{x_i\in\mathcal{S}} \sigma_i \mathbf{g}(x_i) \right] \\
\leq & L\|W^{(0)}\|_2 \left(\frac{1}{|\mathcal{S}|}+\frac{1}{|\mathcal{T}|}\right) \mathbb{E}_{\boldsymbol{\sigma}}\left[ \sup \sum_{x_i\in\mathcal{S}} \sigma_i \|x_i\|_2 \right] \\
\leq & L\|W^{(0)}\|_2 \left(\frac{1}{|\mathcal{S}|}+\frac{1}{|\mathcal{T}|}\right)
\end{aligned}
\end{equation}

In Conclusion,
\begin{equation}
\label{boundR}
\begin{aligned}
\mathfrak{R}_{|\mathcal{S}|+|\mathcal{T}|}(\mathcal{L}, p) \leq \frac{\gamma L\|W^{(0)}\|_2(|\mathcal{S}|+|\mathcal{T}|)}{|\mathcal{S}||\mathcal{T}|}
\end{aligned}
\end{equation}
\end{proof}

In Theorem \ref{bound1}, problem features and algorithm features are uniformly represented in $W^{(0)}$. For ease of understanding, we propose Corollary \ref{corollary1} as follows to explain how problem features and algorithm features affect the value of $\mathfrak{R}_{|\mathcal{S}|+|\mathcal{T}|}(\mathcal{L}, p)$.

\begin{corollary}
\label{corollary1}
Follow the same definition in Theorem \ref{bound1}. Let $x_i\in\mathcal{S}$ denote any training instance with problem feature vector $PF_i$ and algorithm feature vector $AF_i$ in embedding layer. Then the transductive Rademacher complexity of Model$_{a}$, $\mathfrak{R}_{|\mathcal{S}|+|\mathcal{T}|}(\mathcal{L}, p)$, is bounded by:
\begin{equation}
\begin{aligned}
\mathfrak{R}_{|\mathcal{S}|+|\mathcal{T}|}(\mathcal{L}, p) \leq \frac{\gamma L(|\mathcal{S}|+|\mathcal{T}|)}{|\mathcal{S}||\mathcal{T}|}\sup_{x_i\in\mathcal{S}} (\|[PF_i\|_2 + \|[AF_i\|_2)
\end{aligned}
\end{equation}
\end{corollary}

\begin{proof}
The proof can be followed by Theorem \ref{bound1}.
\end{proof}

The upper bound of $\mathfrak{R}_{|\mathcal{S}|+|\mathcal{T}|}(\mathcal{L}, p)$ can be further lead to the bound of generalization error $Error_{|\mathcal{T}|}(\mathbf{f})$ following Lemma \ref{theorem1}:

\begin{lemma}
\label{theorem1}
\cite{el2009transductive} Let $B_1\leq 0$, $B_2\geq 0$ and $\mathcal{V}$ be a (possibly infinite) set of real-valued vectors in $[B_1,B_2]^{m+u}$. Let $B_{max}\triangleq \max(|B_1|,|B_2|)$. Let $q\triangleq \left( \frac{1}{u}+\frac{1}{m} \right)$, $s\triangleq \frac{m+u}{(m+u-\frac{1}{2})(1-\frac{1}{2\max(m,u)})}$ and $c_0 \triangleq \sqrt{\frac{32\ln(4e)}{3}} < 5.05$. Then with probability of at least $1-\delta$
over random permutation $\textbf{\emph{Z}}$ of $I^{m+u}$, for all $v\in \mathcal{V}$,
\begin{equation}
\label{eq_theo1}
\begin{aligned}
\frac{1}{u} \sum_{i=k+1}^{m+u} & v\left(Z_i\right) \leq \frac{1}{m} \sum_{i=1}^k v\left(Z_i\right) + \mathfrak{R}_{m+u}(\mathcal{V}) \\ & +B_{\max } c_0 q \sqrt{\min (m, u)} + B \sqrt{\frac{s}{2} q \ln \frac{1}{\delta}}
\end{aligned}
\end{equation}
\end{lemma}

Lemma \ref{theorem1} could be applied with an appropriate instantiation of the set $\mathcal{V}$ so that $\frac{1}{u} \sum_{i=k+1}^{m+u} v\left(Z_i\right)$ corresponds to the test error and $\frac{1}{m} \sum_{i=1}^k v\left(Z_i\right) + \mathfrak{R}_{m+u}(\mathcal{V})$ to the empirical error. In the following, we propose Theorem \ref{bound2} to derive tight upper bounds on transductive generalization error based on Theorem \ref{bound1}.

\begin{theorem}
\label{bound2}
Let $\mathcal{D}\rightarrow \mathcal{S}\cup\mathcal{T}$ is a partition of random permutation of $\mathcal{D}$ with partition ratio $\eta = \frac{|\mathcal{T}|}{|\mathcal{S}|}<1$ \footnote{The partition ratio $\eta<1$ means that the training set scale is larger than the test set, as a research convention.}, and $\mathcal{S}_\mathcal{A}$, $\mathcal{S}_\mathcal{P}$, $\mathcal{T}_\mathcal{A}$, $\mathcal{T}_\mathcal{P}$, denote the candidate algorithms and problems in $\mathcal{S}$ and $\mathcal{T}$. $L$ is the Lipschitz constants of the model excluding algorithm embedding layer, $\gamma$ is the Lipschitz constant of the loss function, and $W^{(0)}$ is the parameters composed of the problems features and parameters in algorithm embedding layer. $Error_{\mathcal{T}}(\mathbf{f})$ and $Error_{\mathcal{S}}(\mathbf{f})$ denote the test error and empirical error of Model$_{a}$, and $c_0 \triangleq \sqrt{\frac{32\ln(4e)}{3}}$. Then with probability of at least $1-\delta$, for all $\ell\left(\mathbf{f}(x_i), y_i\right)\in \mathcal{L}$,
\begin{equation}
\label{eq_theo1}
\begin{aligned}
Error_{\mathcal{T}}(\mathbf{f}) \leq &Error_{\mathcal{S}}(\mathbf{f}) + \mathcal{O}\left(\frac{\gamma L\|W^{(0)}\|_2 (1+\eta)}{\eta |\mathcal{S}_{\mathcal{A}}| \cdot |\mathcal{S}_{\mathcal{P}}|}\right) \\ & + \mathcal{O}\left(\frac{c_0(1+\eta)+\sqrt{\frac{1}{2}(1+\eta)\ln\frac{1}{\delta}}}{\sqrt{\eta |\mathcal{S}_{\mathcal{A}}| \cdot |\mathcal{S}_{\mathcal{P}}|}}\right) 
\end{aligned}
\end{equation}
\end{theorem}

\begin{proof}
According to Lemma \ref{theorem1}, for a random permutation of sample set with a split into training set $\mathcal{S}$ and test set $\mathcal{T}$, we can take $\ell\left(\mathbf{f}(x_i), y_i\right)$ as $v$ in Eq. (\ref{eq_theo1}), then with
probability of at least $1-\delta$ that
\begin{equation}
\label{eq_error}
\begin{aligned}
\frac{1}{|\mathcal{T}|} \sum_{x_i,y_i\in\mathcal{T}} \ell\left(\mathbf{f}(x_i), y_i\right) \leq \frac{1}{|\mathcal{S}|} \sum_{x_i,y_i\in\mathcal{S}} \ell\left(\mathbf{f}(x_i), y_i\right) \\ + \mathfrak{R}_{|\mathcal{S}|+|\mathcal{T}|}(\mathcal{L}) + c_0 q \sqrt{\min (|\mathcal{S}|, |\mathcal{T}|)} + \sqrt{\frac{s}{2} q \ln \frac{1}{\delta}}
\end{aligned}
\end{equation}
where $\mathcal{L}$ denotes the space of the random variable $\{\ell\left(\mathbf{f}(x_i), y_i\right)\}_{x_i,y_i\in\mathcal{S}\cup\mathcal{T}}$.

Substitute Eq. (\ref{boundR}) into Eq. (\ref{eq_error}), we have:
\begin{equation}
\label{eq_result}
\begin{aligned}
\frac{1}{|\mathcal{T}|} \sum_{x_i,y_i\in\mathcal{T}} \ell\left(\mathbf{f}(x_i), y_i\right) \leq \frac{1}{|\mathcal{S}|} \sum_{x_i,y_i\in\mathcal{S}} \ell\left(\mathbf{f}(x_i), y_i\right) \\ + \frac{\gamma L\|W^{(0)}\|_2(|\mathcal{S}|+|\mathcal{T}|)}{ |\mathcal{S}||\mathcal{T}|} + c_0 q \sqrt{|\mathcal{T}|} + \sqrt{\frac{s}{2} q \ln \frac{1}{\delta}}
\end{aligned}
\end{equation}

Then, we substitute the value of $s$, $q$, and $\eta$ into Eq. (\ref{eq_result}), the three part can be calculated as:
\begin{equation}
\label{eq_result1}
\begin{aligned}
& \frac{\gamma L\|W^{(0)}\|_2(|\mathcal{S}|+|\mathcal{T}|)}{|\mathcal{S}||\mathcal{T}|} = \frac{\gamma L\|W^{(0)}\|_2(1+\eta)}{\eta |\mathcal{S}|} \\
& c_0 q \sqrt{|\mathcal{T}|} = \frac{c_0 (1+\eta)}{\sqrt{\eta |\mathcal{S}|}}\\
& \sqrt{\frac{s}{2} q \ln \frac{1}{\delta}} \approx \sqrt{\frac{(1+\eta) \ln \frac{1}{\delta}}{2|\mathcal{T}|}}
\end{aligned}
\end{equation}
where we approximate $1-\frac{1}{2|\mathcal{S}|}$ and $(2+2\eta)|\mathcal{T}|-\eta$ to $1$ and $(2+2\eta)|\mathcal{T}|$, respectively. Hence, we have
\begin{equation}
\begin{aligned}
&Error_{\mathcal{T}}(\mathbf{f}) \leq Error_{\mathcal{S}}(\mathbf{f}) +  \\ & \frac{\gamma L\|W^{(0)}\|_2(1+\eta)}{\eta |\mathcal{S}|} + \frac{c_0(1+\eta)}{\sqrt{\eta |\mathcal{S}|}} + \sqrt{\frac{(1+\eta) \ln \frac{1}{\delta}}{2|\mathcal{T}|}}
\end{aligned}
\end{equation}
By substituting $|\mathcal{S}|=|\mathcal{S}_{\mathcal{P}}| \cdot |\mathcal{S}_{\mathcal{A}}|$ and $|\mathcal{T}|=|\mathcal{T}_{\mathcal{P}}| \cdot |\mathcal{T}_{\mathcal{A}}|$, we obtain the results in Theorem \ref{bound2}.
\end{proof}

According to Theorem \ref{bound2}, we observe that the generalization error under the usage of adaptive features is related to three factors: the training data $|\mathcal{S}_{\mathcal{P}}| \cdot |\mathcal{S}_{\mathcal{A}}|$, the partition ratio $\eta$, and some parameters in model including $L$, $W^{(0)}$, and $\gamma$. When the size of the training set is large enough such that $|\mathcal{S}| \gg \frac{\gamma^2L^2\|W^{(0)}\|^2_2 (1+\eta)^2}{c^2\eta}$ where $c=c_0(1+\eta)+\sqrt{\frac{1}{2}(1+\eta)\ln\frac{1}{\delta}}$, the trend of the generalization error tends to decrease continuously, mainly following the slack term  $\mathcal{O}\left(\frac{c}{\sqrt{\eta|\mathcal{S}_{\mathcal{A}}| \cdot |\mathcal{S}_{\mathcal{P}}|}}\right)$ as the training scale grows. While if the training data is insufficient, the model parameters and the size of the training set will jointly affect the generalization error, mainly following the slack term $\mathcal{O}\left(\frac{\gamma L\|W^{(0)}\|_2 (1+\eta)}{\eta |\mathcal{S}_{\mathcal{A}}| \cdot |\mathcal{S}_{\mathcal{P}}|}\right)$. This is because the model may fail to capture all the patterns in the data, leading to underfitting. In such cases, it is not advisable to have too many or overly complex model parameters. Based on Theorem \ref{bound2}, it can also be observed that the generalization error is relevant to the product of the norms of each layer's parameter matrix, as determined by the computation of the upper bound for the Lipschitz constant $L$. Because the $L_2$-norm emphasizes the maximum singular value of the parameter matrix \cite{anton2013elementary}, outliers in parameter matrix may have a significant impact on the maximum singular value. This highlights the importance of employing preprocessing techniques such as normalization and outlier detection in practical usage to mitigate the occurrence of outliers in the parameter matrices of neural networks. In our subsequent research within the inductive setting, we will derive a bound represented by Frobenius norm, which will demonstrate enhanced resilience to outliers \cite{bottcher2008frobenius}, providing better control over the upper bound of generalization error.

On the other hand, it is not difficult to observe from Theorem \ref{bound2} that, when using the adaptive algorithm feature, the increase in the number of candidate algorithms will have a positive impact on reducing generalization error, following the slack term $\mathcal{O}\left(\sqrt{\frac{c}{\eta |\mathcal{S}_{\mathcal{A}}| \cdot |\mathcal{S}_{\mathcal{P}}|}}\right)$, which is significantly different from traditional learning approaches that focus on learning the mapping from problems to algorithms. When the algorithm feature is absence, for example, some studies model algorithm selection as a multi-class classification problem or performance regression problem based on problem features. In this case, the derivation process is similar to Theorem \ref{bound2}, which allows us to easily obtain upper bounds for the transductive generalization error of the model under the performance regression approach (Model$_{reg}$) and the multi-class classification approach (Model$_{cla}$). We present them in the form of Corollary \ref{class_regre_bound}:

\begin{corollary}
\label{class_regre_bound}
Follow the same definition in Theorem \ref{bound2}: (1) Let $L_r$ and $\gamma_r$ be the Lipschitz constants of the Model$_{reg}$ and its loss function. $c = c_0(1+\eta)+\sqrt{\frac{1}{2}(1+\eta)\ln\frac{1}{\delta}}$. $Error_{\mathcal{T}}(\text{Model}_{reg})$ and $Error_{\mathcal{S}}(\text{Model}_{reg})$ denote the test error and empirical error of Model$_{reg}$, then with probability of at least $1-\delta$,
\begin{equation}
\label{eq_class_regre_bound1}
\begin{aligned}
& Error_{\mathcal{T}}(\text{Model}_{reg}) \leq Error_{\mathcal{S}}(\text{Model}_{reg}) + \\ &  \mathcal{O}\left(\frac{\gamma_r L_r \|W^{(0)}\|_2 (1+\eta)}{\eta |\mathcal{S}_{\mathcal{P}}|}\right) + \mathcal{O}\left(\frac{c}{\sqrt{\eta |\mathcal{S}_{\mathcal{P}}|}}\right) 
\end{aligned}
\end{equation}
(2) Let $L_c$ and $\gamma_c$ be the Lipschitz constants of the Model$_{cla}$ and its loss function. $c = c_0(1+\eta)+\sqrt{\frac{1}{2}(1+\eta)\ln\frac{1}{\delta}}$. $Error_{\mathcal{T}}(\text{Model}_{cla})$ and $Error_{\mathcal{S}}(\text{Model}_{cla})$ denote the test error and empirical error of Model$_{cla}$, then with probability of at least $1-\delta$,
\begin{equation}
\label{eq_class_regre_bound2}
\begin{aligned}
& Error_{\mathcal{T}}(\text{Model}_{cla}) \leq Error_{\mathcal{S}}(\text{Model}_{cla}) + \\ &  \mathcal{O}\left(\frac{\gamma_m L_m |\mathcal{S}_{\mathcal{A}}| \|W^{(0)}\|_2 (1+\eta)}{\eta |\mathcal{S}_{\mathcal{P}}|}\right) + \mathcal{O}\left(\frac{c}{\sqrt{\eta |\mathcal{S}_{\mathcal{P}}|}}\right) 
\end{aligned}
\end{equation}
\end{corollary}

\begin{proof}
The proof can be followed by Theorem \ref{class_regre_bound}, where the proof process of the Model$_{cla}$ further involves the application of Lemma 1 in \cite{maximov2016tight}, to solve the multi-class scenario.
\end{proof}

Herein, we focus on analyzing the relationship between the upper bounds of generalization error for different models and the size of the training instances. According to Corollary \ref{class_regre_bound}, we find that, in the case of a finite training instance size, the generalization error upper bound of Model$_{reg}$ is represented by a slack term of $\mathcal{O}\left(\frac{\gamma_r L_r \|W^{(0)}\|_2 (1+\eta)}{\eta |\mathcal{S}_{\mathcal{P}}|}\right)$ in relation to the training set. This indicates a weaker generalization capability in terms of the growing number of algorithms, compared to Model$_a$ which exhibits a slack term of $\mathcal{O}\left(\frac{\gamma L\|W^{(0)}\|_2 (1+\eta)}{\eta |\mathcal{S}_{\mathcal{A}}| \cdot |\mathcal{S}_{\mathcal{P}}|}\right)$. Moreover, the upper bound of Model$_{cla}$'s generalization error is even proportional to the number of candidate algorithms, following the slack term $\mathcal{O}\left(\frac{\gamma_c L_c \|W^{(0)}\|_2 |\mathcal{S}_{\mathcal{A}}| (1+\eta)}{\eta|\mathcal{S}_{\mathcal{P}}|}\right)$. These results underscore the advantages of utilizing adaptive algorithm features, particularly when there is a relatively large number of candidate algorithms. The incorporation of adaptive algorithmic features is beneficial for enhancing the model's generalization performance. In the case of abundant training samples, adaptive algorithm features demonstrate the same advantages in generalization performance. The overall asymptotic rate of Model$_{reg}$, Model$_{cla}$, and Model$_a$ follow the slack term $\mathcal{O}\left(\frac{c}{\sqrt{\eta |\mathcal{S}_{\mathcal{P}}|}}\right)$, $\mathcal{O}\left(\frac{c}{\sqrt{\eta |\mathcal{S}_{\mathcal{P}}|}}\right)$, and $\mathcal{O}\left(\frac{c}{\sqrt{\eta |\mathcal{S}_{\mathcal{A}}| \cdot |\mathcal{S}_{\mathcal{P}}|}}\right)$, respectively.

\section{Generalization Error of the Predefined Feature-based Model}

Adaptive algorithm features can only be used in scenarios where the candidate algorithms are the same in the training and test sets. If we want the model to have generalization across both new problems and new algorithms, the model needs to consider both problem features and predefined features of algorithms. The Model$_{b}$ learns the matching relationship between problems and algorithms throughout the entire distribution using inductive learning.

According to Definition \ref{in_R_comp}, the inductive Rademacher Complexity can be calculated as:
\begin{equation}
\begin{aligned}
\hat{\mathfrak{R}}_{|\mathcal{S}|}\left(\mathcal{F}\right) = \frac{1}{|\mathcal{S}|} \mathbf{E}_{\boldsymbol{\sigma}} \sup _{\mathcal{N}_1^{l-1}, W^{(l)}} \sum_{x_i\in\mathcal{S}} \sigma_i W^{(l)} \phi^{(l-1)}\left(\mathcal{N}_1^{l-1}\left(x_i\right)\right)
\end{aligned}
\end{equation}
where $\mathcal{F}$ is a class of real-valued network, $\mathcal{N}_i^{j}$ denotes the sub-network from $i$-th layer to $j$-th layer in model $\mathbf{f}\left(x_i\right)$, $W^{(i)}$ denotes the parameters of the $i$-th layer and $\phi^{(i)}(*)$ denotes the $1$-Lipschitz and positive-homogeneous activation function. To analyze the upper bound of inductive Rademacher complexity $\mathfrak{R}_{|\mathcal{S}|}(\mathcal{F})$, we first provide two auxiliary lemmas in \cite{golowich2018size} and \cite{mcdiarmid1989method}:

\begin{lemma}
\label{lemma1}
\cite{golowich2018size} Let $\phi$ be a $1$-Lipschitz, positive-homogeneous activation function which is applied element-wise. Then for any class of vector-valued functions $\mathcal{F}$, and any convex and monotonically increasing function $g$: $\mathbb{R}\rightarrow[0,\infty)$,
\begin{equation}
\begin{aligned}
&\mathbf{E}_{\boldsymbol{\epsilon}} \sup _{f \in \mathcal{F}, W:\|W\|_F \leq R} g\left(\left\|\sum_{i=1}^m \epsilon_i \phi\left(W f\left(\mathbf{x}_i\right)\right)\right\|\right) \leq \\ & 2 \mathbf{E}_{\boldsymbol{\epsilon}} \sup _{f \in \mathcal{F}} g\left(R \left\|\sum_{i=1}^m \epsilon_i f\left(\mathbf{x}_i\right)\right\|\right)
\end{aligned}
\end{equation}
\end{lemma}

\begin{lemma}
\label{lemma2}
(An intermediate result of the proof of the Bounded Differences Inequality \cite{mcdiarmid1989method}) Assume that the function $\Phi(X_1,X_2, \cdots, X_n)$ satisfies the bounded differences assumption with constants $c_1, c_2, \cdots, c_n$, where $X_i$ are independent. Then, for every $\lambda\in \mathbb{R}$
\begin{equation}
\label{lemma2eq}
\log \mathbf{E} \exp[\lambda(\Phi-\mathbf{E}\Phi)] \leq \frac{\lambda^2\sum^{n}_{i=1}c^2_i}{8}
\end{equation}
\end{lemma}

Lemma \ref{lemma1} will be used in inequality scaling in the proof of Theorem \ref{bound3}, and Lemma \ref{lemma2} will be used to bound the expectation term. Hence, we first introduce a parameter $\lambda$ and the calculation `$\exp$', to construct the form in Eq. (\ref{lemma2eq}), as shown in the proof of Theorem \ref{bound3} as follows.

\begin{theorem}
\label{bound3}
Let $|\mathcal{S}|$ denote the scale of training samples, $R_i$ denote the upper bound of the Frobenius norm of the parameter matrix in $i$-th layer, i.e., $\|W^{(i)}\|_F\leq R_i$, and $l$ denote the number of layers. Then, for the Model$_{b}$ $\mathbf{f}\left(x_i\right)$ with $1$-Lipschitz, positive-homogeneous activation functions, the inductive Rademacher complexity $\hat{\mathfrak{R}}_{|\mathcal{S}|}\left(\mathcal{F}\right)$ of Model$_{b}$ is bounded by:
\begin{equation}
\begin{aligned}
&\hat{\mathfrak{R}}_{|\mathcal{S}|}\left(\mathcal{F}\right) \leq \frac{1}{|\mathcal{S}|} \\ &\sqrt{2l\log 2 \prod_{i=1}^l R_i \sum_{x_j\in \mathcal{S}} \| x_j \|^2 + 2 \left( \prod_{i=1}^l R_i \right)^2 \left( \sum_{x_j\in \mathcal{S}}\| x_j \|^2 \right)^{\frac{3}{2}}}
\end{aligned}
\end{equation}
\end{theorem}

\begin{proof}
According to Jensen's inequality for concave functions and Cauchy-Schwarz inequality, $\hat{\mathfrak{R}}_{|\mathcal{S}|}\left(\mathcal{F}\right)$ satisfies the following inequality:
\begin{equation}
\label{theorem3_1}
\begin{aligned}
\hat{\mathfrak{R}}_{|\mathcal{S}|}\left(\mathcal{F}\right) \leq & \frac{1}{|\mathcal{S}|} \log \left\{ \mathbf{E}_{\boldsymbol{\sigma}} \sup _{\mathcal{N}_1^{l-1}, \|W^{(l)}\|\leq R_l} \exp \right.\\ &\left. \left[\sum_{x_i\in\mathcal{S}} \sigma_i R_l \|\phi^{(l-1)}\left(\mathcal{N}_1^{l-1}\left(x_i\right)\right)\| \right] \right\}
\end{aligned}
\end{equation}
where the $\| W^{(l)} \|$ is equivalently denoted with $\| W^{(l)} \|_F$ to keep the notation consistent with other layers below.

According to Lemma \ref{lemma1}, Eq. (\ref{theorem3_1}) can be further transformed as:
\begin{equation}
\label{theorem3_2}
\begin{aligned}
\hat{\mathfrak{R}}_{|\mathcal{S}|}\left(\mathcal{F}\right) \leq & \frac{1}{|\mathcal{S}|} \log \left\{2 \mathbf{E}_{\boldsymbol{\sigma}} \sup _{\mathcal{N}_1^{l-2}, \|W^{(l-1)}\|_F\leq R_{l-1}} \exp \right.\\ &\left. \left[\sum_{x_i\in\mathcal{S}} \sigma_i R_l R_{l-1} \|\phi^{(l-2)}\left(\mathcal{N}_1^{l-2}\left(x_i\right)\right)\| \right] \right\}
\end{aligned}
\end{equation}
Repeating this transformation $l$ times, we have:
\begin{equation}
\label{eq_CalR}
\hat{\mathfrak{R}}_{|\mathcal{S}|}\left(\mathcal{F}\right) \leq \frac{1}{|\mathcal{S}|} \log \left[2^l \mathbf{E}_{\boldsymbol{\sigma}} \exp \left(\prod_{i=1}^l R_i \left\|\sum_{x_j\in\mathcal{S}} \sigma_j x_j\right\|\right)\right]
\end{equation}
where $\|W^(i)\|_F\leq R_i$. Taking the $\Phi(\boldsymbol{\sigma})=\prod_{i=1}^l R_i \left\|\sum_{x_j\in\mathcal{S}} \sigma_j x_j\right\|$ as the function of $\boldsymbol{\sigma}$, then we can prove the bounded differences of function $\Phi(\boldsymbol{\sigma})$ as follow: for $\forall 1 \leq j \leq |\mathcal{S}|$ and a value of $\boldsymbol{\sigma}$ denoted as $\sigma=\{\sigma_j\}_{j=1}^{|\mathcal{S}|}$, we have
\begin{equation}
\begin{aligned}
& \sup |\Phi(\sigma)-\Phi(\sigma/\{\sigma_j\}\cup\{\sigma_j^{\prime}\})|\\
=&\Phi(\sigma)-\Phi(\sigma/\{1\}\cup\{-1\}) \leq 2 \| x_j \| \prod_{i=1}^l R_i
\end{aligned}
\end{equation}
By Lemma \ref{lemma2}, we can introduce an arbitrary parameter $\lambda\in\mathbb{R}$ into Eq. (\ref{eq_CalR}) and bound the $\mathbf{E}_{\boldsymbol{\sigma}} \{\exp \lambda(\Phi(\boldsymbol{\sigma})-\mathbf{E}_{\boldsymbol{\sigma}} \Phi(\boldsymbol{\sigma}))\}$ as:
\begin{equation}
\label{eq_E1}
\begin{aligned}
\log \mathbf{E}_{\boldsymbol{\sigma}} \{\exp \lambda(\Phi(\boldsymbol{\sigma})-\mathbf{E}_{\boldsymbol{\sigma}} \Phi(\boldsymbol{\sigma}))\} \leq \frac{1}{2} \lambda^2 \prod_{i=1}^l R_i \sum_{x_j\in \mathcal{S}} \| x_j \|^2
\end{aligned}
\end{equation}
Then, we solve the value of $\mathbf{E}_{\boldsymbol{\sigma}}\Phi(\boldsymbol{\sigma})$. According to the property of Rademacher random variables ($\mathbf{E}_{\boldsymbol{\sigma}} [\sigma_i] = 1$) and the Jensen's inequality, we have:
\begin{equation}
\label{eq_E2}
\begin{aligned}
\mathbf{E}_{\boldsymbol{\sigma}}\Phi(\boldsymbol{\sigma}) & = \prod_{i=1}^l R_i \mathbf{E}_{\boldsymbol{\sigma}} \left[\left\|\sum_{x_j\in \mathcal{S}} \sigma_j x_j\right\|\right] \\
& \leq \prod_{i=1}^l R_i \sqrt{\mathbf{E}_{\boldsymbol{\sigma}} \left[ \left\| \sum_{x_j\in \mathcal{S}}  \sigma_j x_j\right\|^2 \right]} \\ & \leq \prod_{i=1}^l R_i \sqrt{\sum_{x_j\in \mathcal{S}} \|x_j\|^2}
\end{aligned}
\end{equation}
Substitute Eq. (\ref{eq_E1}) and Eq. (\ref{eq_E2}) into Eq. (\ref{eq_CalR}), the upper bound of $\hat{\mathfrak{R}}_{|\mathcal{S}|}\left(\mathcal{F}\right)$ could be calculated as:
\begin{equation}
\label{theo3_1}
\begin{aligned}
\hat{\mathfrak{R}}_{|\mathcal{S}|}\left(\mathcal{F}\right) & \leq \frac{1}{|\mathcal{S}|\lambda} \log \{ 2^l \cdot \mathbf{E}_{\boldsymbol{\sigma}} \{ \exp \lambda [\Phi(\boldsymbol{\sigma})-\mathbf{E}_{\boldsymbol{\sigma}} \Phi(\boldsymbol{\sigma})] \\
& \exp \lambda \mathbf{E}_{\boldsymbol{\sigma}} \Phi(\boldsymbol{\sigma}) \} \} \\
& \leq \frac{1}{|\mathcal{S}|\lambda} \left( l\log2+\prod_{i=1}^l R_i \sqrt{\sum_{x_j\in \mathcal{S}} \|x_j\|^2} \right) + \\ &\frac{\lambda}{2|\mathcal{S}|}\prod_{i=1}^l R_i \sum_{x_j\in \mathcal{S}} \| x_j \|^2
\end{aligned}
\end{equation}
Since the Eq. (\ref{theo3_1}) holds for any value of $\lambda$, solving for the minimum value of the equation will give the tightest upper bound for $\hat{\mathfrak{R}}_{|\mathcal{S}|}\left(\mathcal{F}\right)$, i.e.,
\begin{equation}
\label{theo3_2}
\begin{aligned}
&\hat{\mathfrak{R}}_{|\mathcal{S}|}\left(\mathcal{F}\right) \leq \frac{1}{|\mathcal{S}|} \\ &\sqrt{2l\log 2 \prod_{i=1}^l R_i \sum_{x_j\in \mathcal{S}} \| x_j \|^2 + 2 \left( \prod_{i=1}^l R_i \right)^2 \left( \sum_{x_j\in \mathcal{S}}\| x_j \|^2 \right)^{\frac{3}{2}}}
\end{aligned}
\end{equation}
\end{proof}

The upper bound of $\hat{\mathfrak{R}}_{|\mathcal{S}|}\left(\mathcal{F}\right)$ can further lead to the bound of generalization error following Lemma \ref{theorem3}:

\begin{lemma}
\label{theorem3}
\cite{kakade2008complexity} Consider an arbitrary function class $\mathcal{F}$ such that $\forall \mathbf{f}\in \mathcal{F}$ we have $\sup_{\mathbf{x}\in\mathcal{X}} |\mathbf{f}(x)| \leq C$. Then, with probability at least $1-\delta$ over the sample, for all margins $\gamma > 0$ and all $\forall \mathbf{f}\in \mathcal{F}$ we have,
\begin{equation}
\mathcal{L}(\mathbf{f}) \leq K_\gamma(\mathbf{f})+4 \frac{\mathfrak{R}_n(\mathcal{F})}{\gamma}+\sqrt{\frac{\log \left(\log _2 \frac{4 C}{\gamma}\right)}{n}}+\sqrt{\frac{\log (1 / \delta)}{2 n}}
\end{equation}
where $K_\gamma(\mathbf{f})$ denotes the fraction of the data having $\gamma-$margin mistakes, i.e., $K_\gamma(\mathbf{f}):=\frac{\left|\left\{i: y_i \mathbf{f}\left(\mathbf{x}_i\right)<\gamma\right\}\right|}{n}$.
\end{lemma}

Based on the upper bound of $\hat{\mathfrak{R}}_{|\mathcal{S}|}\left(\mathcal{F}\right)$ and Lemma \ref{theorem3}, we discuss the generalization performance of model with predefined algorithm features in Theorem \ref{theorem4}, as well as Corollary \ref{theorem4_col}, describing a common scenarios in practical applications.

\begin{theorem}
\label{theorem4}
Let $Error_{\mathcal{T}}(\mathbf{f})$ and $Error_{\mathcal{S}}(\mathbf{f})$ denote the test error and empirical error of the Model$_{b}$ $\mathbf{f}\left(x_i\right)$, $|\mathcal{S}_{\mathcal{P}}|$ and $|\mathcal{S}_{\mathcal{A}}|$ denote the number of problems and algorithms in training samples, $R_i$ denote the upper bound of the Frobenius norm of the parameter matrix in $i$-th layer, i.e., $\|W^{(i)}\|_F\leq R_i$, and $l$ denote the number of layers. Then, with probability at least $1-\delta$ over the sample, for all margins $\gamma > 0$ and all $\forall \mathbf{f}\in \mathcal{F}$ with $1$-Lipschitz, positive-homogeneous activation functions, we have,
\begin{equation}
\label{eq_theo4}
\begin{aligned}
&Error_{\mathcal{T}}(\mathbf{f}) \leq Error_{\mathcal{S}}(\mathbf{f}) + \mathcal{O}\left(\frac{c_1}{\sqrt{|\mathcal{S}_{\mathcal{P}}| \cdot |\mathcal{S}_{\mathcal{A}}|}}\right) + \\& \mathcal{O}\left(\frac{\sqrt{ (\log 2)l\Gamma_{\mathbf{f}}\Gamma_{\mathcal{S}} + \Gamma^2_{\mathbf{f}}\Gamma^{\frac{3}{2}}_{\mathcal{S}}|\mathcal{S}_{\mathcal{P}}|^\frac{1}{2} |\mathcal{S}_{\mathcal{A}}|^\frac{1}{2} }}{\sqrt{|\mathcal{S}_{\mathcal{P}}| \cdot |\mathcal{S}_{\mathcal{A}}|}\gamma}\right)
\end{aligned}
\end{equation}
where $\Gamma_{\mathbf{f}}$ and $\Gamma_{\mathcal{S}}$ are model-related and data-related variables, and $c_1$ is a constant, denoted as:
\begin{equation}
\begin{aligned}
& \Gamma_{\mathbf{f}} = \prod_{i=1}^l R_i,\\
& \Gamma_{\mathcal{S}} = \max_{x_j\in \mathcal{S}}\| x_j \|^2, \\
& c_1 = \sqrt{ \log \left(\log _2 \frac{4}{\gamma}\right) } + \sqrt{ \log (\frac{1}{\delta}) }
\end{aligned}
\end{equation}
\end{theorem}

\begin{proof}
According to the Lemma \ref{theorem3}, the generalization performance on test samples can be bounded by the margin loss on training samples, with probability at least $1-\delta$ for $\forall \mathbf{f}\in \mathcal{F}$:
\begin{equation}
\begin{aligned}
P_{\mathcal{T}}(y_i \mathbf{f}\left(x_i\right)<0) \leq  \frac{1}{|\mathcal{S}|}\left|\left\{x_i,y_i\in\mathcal{S}: y_i \mathbf{f}\left(x_i\right)<\gamma\right\}\right| \\ +  \frac{4 \mathfrak{R}_{|\mathcal{S}|}(\mathcal{F})}{\gamma} +\sqrt{\frac{\log \left(\log _2 \frac{4 C}{\gamma}\right)}{|\mathcal{S}|}}+\sqrt{\frac{\log (\frac{1}{\delta})}{2 |\mathcal{S}|}}
\end{aligned}
\end{equation}
where $C = 1$ and $\mathfrak{R}_{|\mathcal{S}|}(\mathcal{F})$ is bounded by Theorem \ref{bound3}. Hence,
\begin{equation}
\label{theo4_1}
\begin{aligned}
Error_{\mathcal{T}}(\mathbf{f}) \leq & Error_{\mathcal{S}}(\mathbf{f}) + \frac{4\sqrt{2}}{\sqrt{|\mathcal{S}|}\gamma} \sqrt{ (\log 2)l\Gamma_{\mathbf{f}}\Gamma_{\mathcal{S}} + \Gamma^2_{\mathbf{f}}\Gamma^{\frac{3}{2}}_{\mathcal{S}} |\mathcal{S}|^\frac{1}{2} } \\& + \frac{1}{\sqrt{|\mathcal{S}|}}\left(\sqrt{ \log \left(\log _2 \frac{4}{\gamma}\right) } + \sqrt{ \log (\frac{1}{\delta}) }\right)
\end{aligned}
\end{equation}
By substituting $|\mathcal{S}|=|\mathcal{S}_{\mathcal{P}}| \cdot |\mathcal{S}_{\mathcal{A}}|$, we obtain the results in Theorem \ref{theorem4}.
\end{proof}

\begin{corollary}
\label{theorem4_col}
Follow the same definition in Theorem \ref{theorem4}. If $\Gamma_{\mathbf{f}}, \Gamma_{\mathcal{S}} > 1$ and $\sqrt{|\mathcal{S}_{\mathcal{P}}| \cdot |\mathcal{S}_{\mathcal{A}}|} \gg (\log 2)l$, then with probability at least $1-\delta$ over the sample, we have,
\begin{equation}
\begin{aligned}
Error_{\mathcal{T}}(\mathbf{f}) & \leq Error_{\mathcal{S}}(\mathbf{f}) + \\& \mathcal{O}\left(\frac{\Gamma_{\mathbf{f}}\Gamma_{\mathcal{S}}^{\frac{3}{4}}}{|\mathcal{S}_{\mathcal{P}}|^{\frac{1}{4}} |\mathcal{S}_{\mathcal{A}}|^{\frac{1}{4}}\gamma}\right) + \mathcal{O}\left(\frac{c_1}{\sqrt{|\mathcal{S}_{\mathcal{P}}| \cdot |\mathcal{S}_{\mathcal{A}}|}}\right)
\end{aligned}
\end{equation}
\end{corollary}

\begin{proof}
The proof can be followed by Theorem \ref{theorem4} with conditions $\Gamma_{\mathbf{f}}, \Gamma_{\mathcal{S}} > 1$ and $\sqrt{|\mathcal{S}_{\mathcal{P}}| \cdot |\mathcal{S}_{\mathcal{A}}|} \gg (\log 2)l$.
\end{proof}

According to Theorem \ref{theorem4}, we can observe that the inductive generalization error is related to three factors: the training scale $|\mathcal{S}|$, the model parameter-related factor $\Gamma_{\mathbf{f}}$ and loss function parameter $\gamma$, and the value of training data $\Gamma_{\mathcal{S}}$. By Eq. (\ref{eq_theo4}), in most cases where $\Gamma_{\mathbf{f}}, \Gamma_{\mathcal{S}} > 1$ and $\sqrt{|\mathcal{S}|} \gg (\log 2)l$, Theorem \ref{theorem4} can be simplified to the conclusion in Corollary \ref{theorem4_col}, which states that the rate of change is governed by $\mathcal{O}\left(\frac{\Gamma_{\mathbf{f}}\Gamma_{\mathcal{S}}^{\frac{3}{4}}}{|\mathcal{S}_{\mathcal{A}}|^{\frac{1}{4}} |\mathcal{S}_{\mathcal{P}}|^{\frac{1}{4}}\gamma}\right)$. In rare cases where $\Gamma_{\mathbf{f}}, \Gamma_{\mathcal{S}} < 1$ or $\sqrt{|\mathcal{S}|} \approx (\log 2)l$, the rate of change in generalization error becomes $\mathcal{O}\left(\frac{\sqrt{ (\log 2)l\Gamma_{\mathbf{f}}\Gamma_{\mathcal{S}}}}{\sqrt{|\mathcal{S}_{\mathcal{A}}|\cdot |\mathcal{S}_{\mathcal{P}}|}\gamma}\right)+ \mathcal{O}\left(\frac{c_1}{\sqrt{|\mathcal{S}_{\mathcal{A}}|\cdot |\mathcal{S}_{\mathcal{P}}|}}\right)$. This means that in scenarios with fewer training instances and a relatively simple model, the improvement of model generalization by training samples is more significant.

Different from the transductive settings, the generalization error in Theorem \ref{theorem4} is not affected by the size of the test set, which is a characteristic brought by inductive learning paradigm. Furthermore, in inductive manner, the increased number of candidate algorithms also has a positive impact on reducing generalization error, which is similar to the transductive generalization performance. According to the multiplication invariance of the $L_2$-norm in Eq. (\ref{eq_Lipschitz}), both Theorem \ref{bound2} and Theorem \ref{theorem4} are affected by the product of the norms of each layer's parameter matrix, while the generalization performance of intuitive learning is influenced by the Frobenius norm, and the generalization performance of transductive learning is influenced by the $L_2$-norm. 
Additionally, within the inductive setting, we observe a more pronounced difference on relationship between generalization ability and training instances scale compared to transductive settings, at a rate of $\frac{1}{|\mathcal{S}_{\mathcal{P}}|^{\frac{1}{4}} |\mathcal{S}_{\mathcal{A}}|^{\frac{1}{4}}}$. However, due to their distinct generalization implications, these two boundaries should not be directly comparable. The superior asymptotic rate in transductive learning arises from the assumption that the test and training sets originate from the same distribution. We will illustrate in a corollary of Theorem \ref{theorem4} how distribution shift impacts generalization error in inductive learning.


Since the generalization error in the inductive setting pertains specifically to the utilization of predefined features, the subsequent analysis primarily provides some practical advice for incorporating predefined features. Firstly, in terms of the impact of instance value on the generalization performance, adaptive features are superior to predefined features because adaptive features are learned by the model and their values are more stable. Therefore, the influence of $\|W^{(0)}\|_2$ on the generalization error in Theorem \ref{bound2} is limited, while the value of $\Gamma_{\mathcal{S}}$ in Theorem \ref{theorem4} is unpredictable. This indicates that the generalization performance of the model under predefined features should consider the stability and value range of the predefined features. Take large language model-extracted features as an example. In practical applications, these factors should be considered in the selection of pretrained models and the preprocessing of features. Compared to using adaptive features, another advantage of predefined features is their generalization on new algorithms. Below, we simultaneously consider the distribution shift between training and test sets, as well as the impact of negative sampling, to obtain a corollary about generalization from Theorem \ref{theorem4}.

\begin{corollary}
\label{theorem5}
Follow the same definition in Theorem \ref{theorem4}, and $\mathcal{T}_{\mathcal{A}} - \mathcal{S}_{\mathcal{A}} \neq \varnothing$. Let $\chi^2(P_{\mathcal{T}} \| P_{\mathcal{S}}) = \int \frac{P^2_{\mathcal{T}}(x_j)}{P_{\mathcal{S}}(x_j)} - 1$ denote the chi-square divergence between training and test distributions. If $\prod_{i=1}^l R_i > 1$ and $\sqrt{|\mathcal{S}_{\mathcal{P}}| \cdot |\mathcal{S}_{\mathcal{A}}|} \gg (\log 2)l$, then with probability at least $1-\delta$ over the sample, we have,
\begin{equation}
\begin{aligned}
&Error_{\mathcal{T}}(\mathbf{f}) \leq Error_{\mathcal{S}}(\mathbf{f}) + \\& \mathcal{O}\left( \frac{[\chi^2(P_{\mathcal{T}} \| P_{\mathcal{S}}) + 1]^\frac{3}{4} \prod_{i=1}^l R_i \max_{x_j\in \mathcal{S}}\| x_j \|^\frac{3}{2}}{|\mathcal{S}_{\mathcal{P}}|^{\frac{1}{4}} |\mathcal{S}_{\mathcal{A}}|^{\frac{1}{4}}\gamma} \right). 
\end{aligned}
\end{equation}
\end{corollary}

\begin{proof}
When the distribution of training data and test data is different, we introduce the importance rate of each training instance into Theorem \ref{theorem4} according to the test and training probability distribution. Following Eq. (\ref{theo4_1}), we replace $\Gamma_{\mathcal{S}}$ with $\hat{\Gamma}_{\mathcal{S}}$ and calculate the mean on the training set:
\begin{equation}
\hat{\Gamma}_{\mathcal{S}} = \frac{1}{|\mathcal{S}|}\sum_{x_j\in \mathcal{S}} \left(\frac{P_{\mathcal{T}}(x_j)}{P_{\mathcal{S}}(x_j)} \| x_j \|\right)^2
\end{equation}
According to the law of large numbers, the sample mean converges to the expected value as the sample size approaches infinity. Here, we make the following approximation:
\begin{equation}
\hat{\Gamma}_{\mathcal{S}} \approx \lim_{|\mathcal{S}|\rightarrow +\infty} \hat{\Gamma}_{\mathcal{S}} = \mathbf{E}_{x_j \sim P_{\mathcal{S}}} \left[ \frac{P^2_{\mathcal{T}}(x_j)}{P^2_{\mathcal{S}}(x_j)} \| x_j \|^2 \right]
\end{equation}
Further calculations can be performed to obtain the following result:
\begin{equation}
\begin{aligned}
\hat{\Gamma}_{\mathcal{S}}\approx [\chi^2(P_{\mathcal{T}} \| P_{\mathcal{S}}) + 1] \max_{x_j\in \mathcal{S}}\| x_j \|^2
\end{aligned}
\end{equation}
where $\chi^2(P_{\mathcal{T}} \| P_{\mathcal{S}}) = \int \frac{P^2_{\mathcal{T}}(x_j)}{P_{\mathcal{S}}(x_j)} - 1$ is the chi-square divergence. As a result,
\begin{equation}
\begin{aligned}
&Error_{\mathcal{T}}(\mathbf{f}) \leq Error_{\mathcal{S}}(\mathbf{f}) + \\& \mathcal{O}\left(\frac{\sqrt{ (\log 2)l\Gamma_{\mathbf{f}}\hat{\Gamma}_{\mathcal{S}} + \Gamma^2_{\mathbf{f}}\hat{\Gamma}^{\frac{3}{2}}_{\mathcal{S}}|\mathcal{S}|^{\frac{1}{2}} }}{\sqrt{|\mathcal{S}|}\gamma}\right) + \mathcal{O}\left(\frac{c_1}{\sqrt{|\mathcal{S}|}}\right)
\end{aligned}
\end{equation}
where $\Gamma_{\mathbf{f}}$ and $c_1$ follows the definition in Theorem \ref{theorem4}.
\end{proof}

Compared to adaptive features, predefined features demonstrate generalization on new algorithms that are not present in the training data. Corollary \ref{theorem5} explains the impact of distribution shift between training and test data on upper bound of the generalization error. It can be observed that as the chi-square divergence between the distributions increases, the upper bound of the generalization error gradually increases at a rate of $[\chi^2(P_{\mathcal{T}} \| P_{\mathcal{S}}) + 1]^\frac{3}{4}$. Similarly, the size of the problems and candidate algorithms in the training data still affect the generalization error at a rate of $\frac{1}{|\mathcal{S}_{\mathcal{P}}|^{\frac{1}{4}} |\mathcal{S}_{\mathcal{A}}|^{\frac{1}{4}}}$. In addition to the impact of sample size, it has been noted that in scenarios involving distribution shift, the complexity of a model exacerbates generalization error. Thus, when handling distribution shift in practical applications, especially when aiming for generalization capacity on candidate algorithms, it is advisable to refrain from employing excessively complex models.

\section{Experiment}
\label{experiment}

The mathematical bounds, though often too loose for practical applications, can provide a rough estimate of the sample size required for a certain level of performance. This section presents experimental results aimed at validating some of the theoretical results proposed in this paper. In order to examine the influence of changes in distribution and quantity on model performance, it is necessary to adjust the distribution of training and testing data during the theoretical validation process. Hence, the experiments in this study are conducted on simulated data, which provided a controlled setting where the exact underlying distributions of problems and algorithms could be known. In Section \ref{ex1}, we describe the process of generating simulated data for continuous optimization problems. Then, Sections \ref{ex2} and \ref{ex3} investigate how varying the number of training problem instances and candidate algorithms affects model performance, respectively. In Section \ref{ex4}, we explore model performance under distribution shifts in training and testing data, and in Section \ref{ex5}, we analyze how increasing training sample size influences model performance under different distribution shifts. Finally, Section \ref{ex6} looks at the relationship between model complexity and performance, especially under distribution shifts, to determine the optimal model size.

\subsection{Data Simulation}
\label{ex1}

The experiments in this paper are conducted in the context of continuous optimization problems, with the objective of automatically selecting the most optimal meta-heuristic optimization algorithm based on a combination of problem features and algorithm characteristics. Mathematically, the problem scenario can be described as follows:
\begin{equation}
\begin{aligned}
\text { min } & f(\mathbf{x}) \\
\text { s. t. } & \mathbf{x} \in \Omega \subseteq \mathbb{R}^d
\end{aligned}
\end{equation}
where $f$ is the objective function of the problem, $\mathbf{x} = \left( x_1, \cdots , x_d \right)$ denotes the decision vector, and $d$ is the dimension of the decision variable space $\Omega$.

In the process of data generation, each problem instance $f$ is composed of a diverse set of operands $\mathbf{x}$ and operators in $f$, varying in quantity and type. The problem features are represented by the Reverse Polish Notation \cite{hamblin1962translation} expression of the objective function $f$. Specifically, the objective function of the optimization problem is initially structured as a tree, and the relationship between operators and operands within the tree can be derived by traversing the tree structure, thereby constituting the problem's distinctive features. For tree traversal, we employ the post-order traversal technique to convert the expression into a postfix notation, known as the Reverse Polish Notation. By representing the objective function as a tree and performing post-order traversal, we can extract discriminative features from the objective function of the continuous optimization problem. This methodology provides an automated approach for understanding and exploring the underlying structure of the objective function. Post-order traversal enables us to access and extract information such as values, operators, and variables at each node of the tree. As a result, we obtain both the mathematical expression of the objective function and preserve its hierarchical structure.

This paper employs the training data collection strategy proposed in \cite{tian2020recommender}. Firstly, a training set of benchmark problems is generated by randomly creating problems of a predetermined size. All problems in the set have the same number of decision variables and consistent upper and lower limits. The distribution of the problem space is determined by the occurrence and selection probabilities of each operator. Subsequently, multiple metaheuristic algorithms are applied to each problem in order to obtain labels for each sample. Each algorithm is configured with a population of sufficient size and an adequate number of iterations to ensure convergence. The distribution of algorithms in the algorithm space is determined by the number of algorithms and the proportion that demonstrate optimal performance. For each problem instance in the given problem set, we independently execute each candidate algorithm in the algorithm set 20 times, comparing the average of the minimum objective values obtained across these 20 runs to determine the best algorithm for each problem.

\subsection{Impact of the Number of Problems}
\label{ex2}

When discussing generalization error, the influence of training sample size on generalization performance is one of the most important factors to consider. In all theoretical conclusions of this paper, the impact of problem size and algorithm size on the upper bound of generalization error can be observed to varying degrees. Therefore, this section first investigates the effect of problem size growth on model performance under different modeling methods and algorithm feature types. We compare four methods, including two that utilize algorithm features (adaptive features and predefined features) and two that do not use algorithm features (regression and classification modeling approaches). Since this subsection primarily considers the generalization in the transductive learning setting, the distribution of the training and test sets is set to be the same. All models adopt a multi-layer perceptron structure, with the number of layers and neurons per layer determined through grid search to obtain the optimal results.

\begin{figure}[t]
\centering%
\subfigure[Accuracy with 10 Algorithms]{ \centering
    \label{Algorithm10}
    \includegraphics[height=1.5in,width=1.65in]{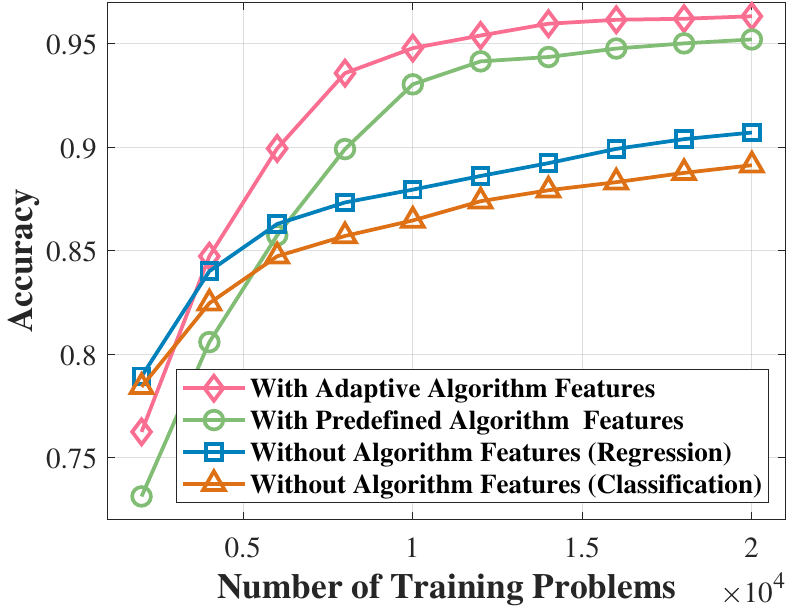}}
\subfigure[Accuracy with 5 Algorithms]{ \centering
    \label{Algorithm5}
    \includegraphics[height=1.5in,width=1.65in]{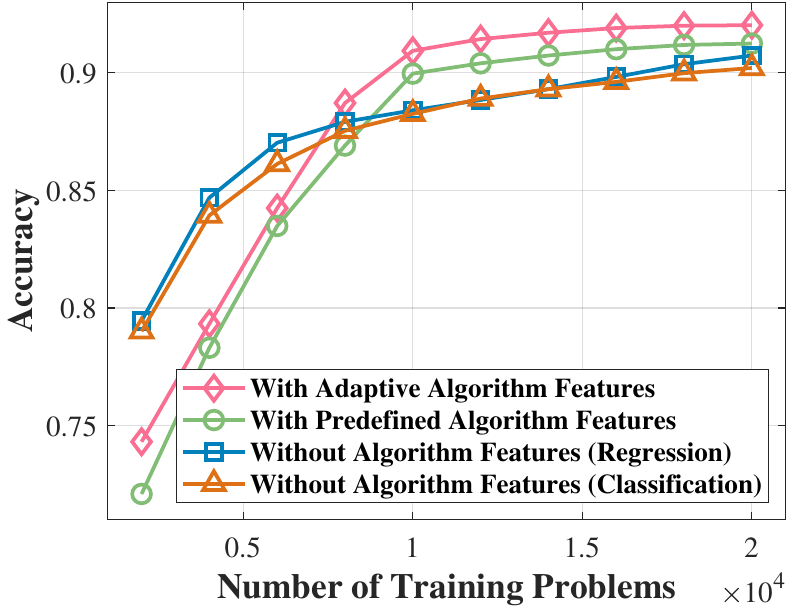}}
\subfigure[Generalization Error with 10 Algorithms]{ \centering
    \label{Algorithm10GE}
    \includegraphics[height=1.5in,width=1.65in]{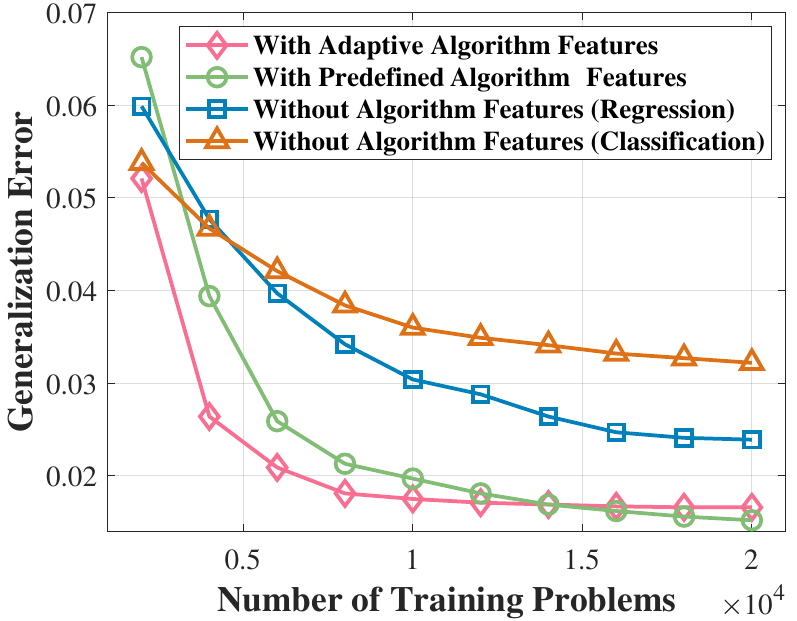}}
\subfigure[Generalization Error with 5 Algorithms]{ \centering
    \label{Algorithm5GE}
    \includegraphics[height=1.5in,width=1.65in]{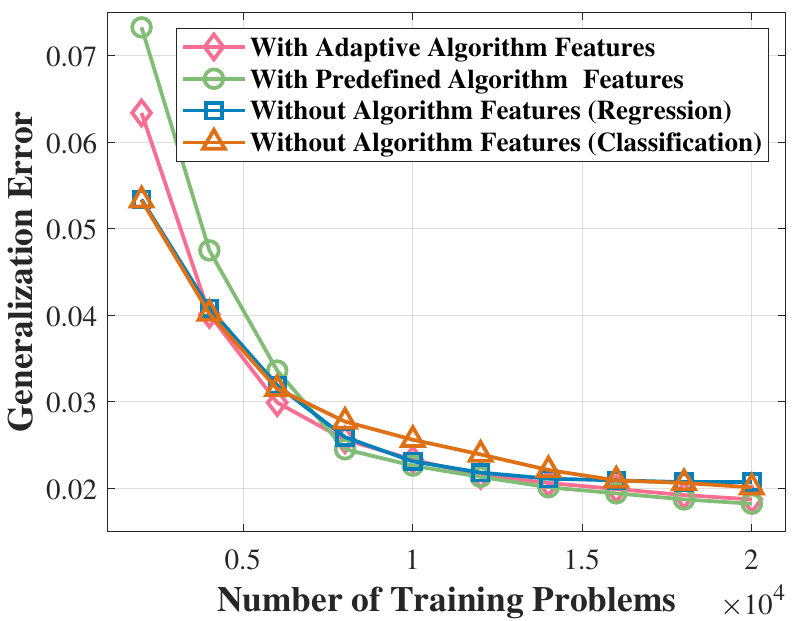}}
\caption{The impact of the number of problem instances on model performance.}
\label{ProblemScale}
\end{figure}

We present the performance curves and error curves of these four methods as the number of problem instances varies for five algorithms and ten algorithms, as shown in Fig. \ref{ProblemScale}. We observe that as the number of problem instances increases, the accuracy of all methods on the test set consistently improves, and the generalization error consistently decreases, albeit at a diminishing rate. This trend aligns with the results of our theoretical analysis. According to Theorem \ref{bound2} and Corollary \ref{class_regre_bound}, when the size of training problem instance is small, the upper bound of generalization error with algorithm features decreases at a rate of $\frac{1}{|\mathcal{S}_{\mathcal{A}}| \cdot |\mathcal{S}_{\mathcal{P}}|}$, while the regression modeling approach and classification modeling approach without algorithm features decrease at rates of $\frac{1}{|\mathcal{S}_{\mathcal{P}}|}$ and $\frac{|\mathcal{S}_{\mathcal{A}}|}{|\mathcal{S}_{\mathcal{P}}|}$, respectively. The number of algorithms $|\mathcal{S}_{\mathcal{A}}|$ has a positive impact on models utilizing algorithm features, no impact on regression models, and a negative impact on classification models. Reflected in the experimental results, models using algorithm features demonstrate a clear advantage, while regression models and classification models perform relatively poorly. Comparing the results in Fig. \ref{Algorithm10} and Fig. \ref{Algorithm5}, as well as Fig. \ref{Algorithm10GE} and Fig. \ref{Algorithm5GE}, we find that the advantage of the method utilizing algorithm features in terms of rate of change becomes more pronounced as the number of algorithms increases, further confirming that one of the benefits of algorithm features is the performance gain with an increasing number of algorithms. In the case of abundant training samples, algorithm features also demonstrate the advantages in generalization performance. The overall asymptotic rate of regression modeling manner, classification modeling manner, and algorithm feature-based modeling manner follow the slack term $\mathcal{O}\left(\frac{c}{\sqrt{\eta |\mathcal{S}_{\mathcal{P}}|}}\right)$, $\mathcal{O}\left(\frac{c}{\sqrt{\eta |\mathcal{S}_{\mathcal{P}}|}}\right)$, and $\mathcal{O}\left(\frac{c}{\sqrt{\eta |\mathcal{S}_{\mathcal{A}}| \cdot |\mathcal{S}_{\mathcal{P}}|}}\right)$, respectively. Similarly, the slopes of all curves in Fig. \ref{ProblemScale} gradually diminish and converge. Models based on regression and classification, in comparison to the utilization of algorithm features, reach the turning point of the curves later and demonstrate inferior overall performance.

\subsection{Impact of the Number of Algorithms}
\label{ex3}

In Section 5.2, we have examined the influence of the number of problem instances on model performance and generalization error, as well as the intervention of candidate algorithm size in this context. To provide a more comprehensive illustration of the impact of candidate algorithm size, this subsection investigates the trend of model performance and generalization error with varying numbers of candidate algorithms. The number of problem instances in the experiment is set to 10,000 and 20,000, while the remaining experimental settings remain consistent with the previous subsection. By gradually increasing the size of the candidate algorithm set in the simulated data, the performance and generalization error changes of each model are depicted in Fig. \ref{AlgorithmScale}. It can be observed that when algorithm features are employed, the model performance generally improves as the number of algorithms increases, and the generalization error consistently decreases. However, in the case of 10,000 problem instances, the accuracy of predefined feature-based model exhibits a different trend when the number of algorithms is small, as shown in Fig. \ref{Problem1W}. This discrepancy can be attributed to the limited training sample size. Nevertheless, this phenomenon does not manifest in the case of 20,000 problem instances in Fig. \ref{Problem2W}. Conversely, models based on classification and regression approaches are adversely affected by an increased number of candidate algorithms, with classification models experiencing a more pronounced negative impact compared to regression models. These trends align with the theoretical conclusions presented in this paper, particularly when an ample number of problem examples are available. From the experiments conducted in this section, it becomes evident that algorithm features find their application in scenarios involving a larger number of candidate algorithms, as models incorporating algorithm features demonstrate superior suitability in such contexts.

\begin{figure}[t]
\centering%
\subfigure[Accuracy with 10,000 Problems]{ \centering
    \label{Problem1W}
    \includegraphics[height=1.5in,width=1.65in]{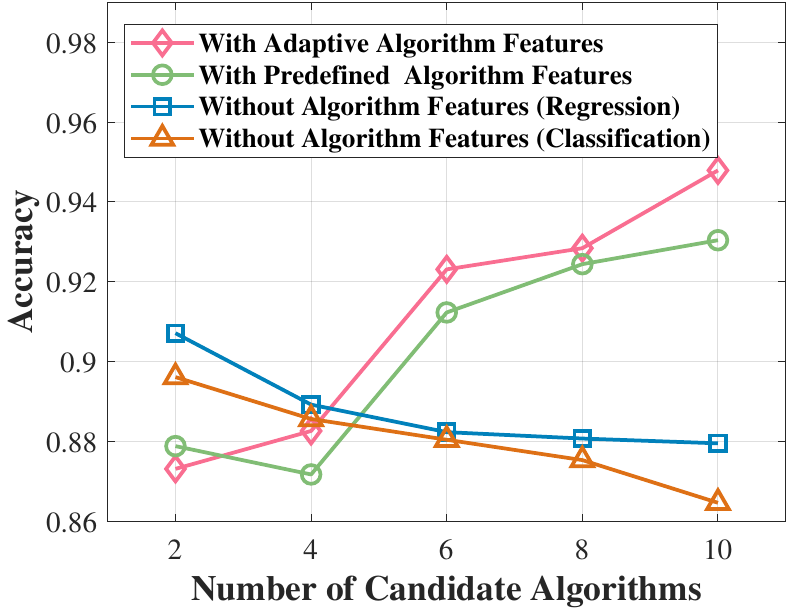}}
\subfigure[Accuracy with 20,000 Problems]{ \centering
    \label{Problem2W}
    \includegraphics[height=1.5in,width=1.65in]{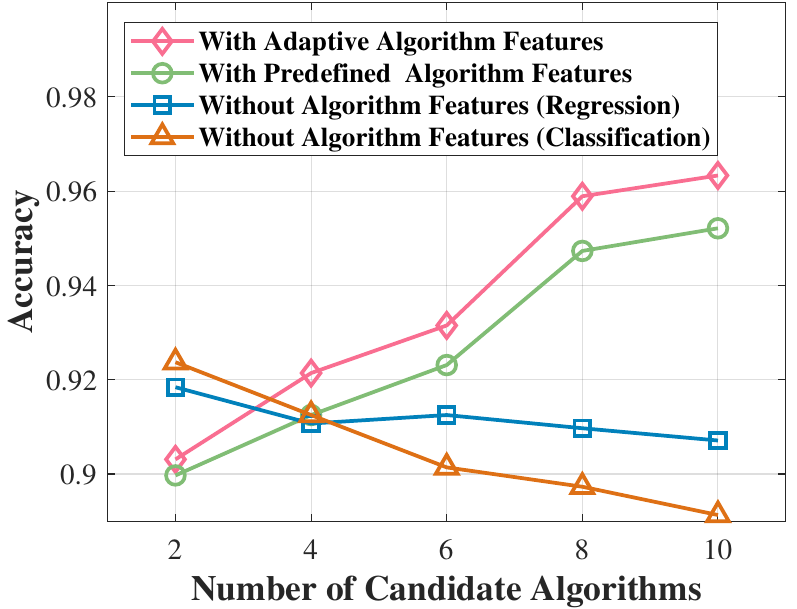}}
\subfigure[Generalization Error with 10,000 Problems]{ \centering
    \label{Problem1WGE}
    \includegraphics[height=1.5in,width=1.65in]{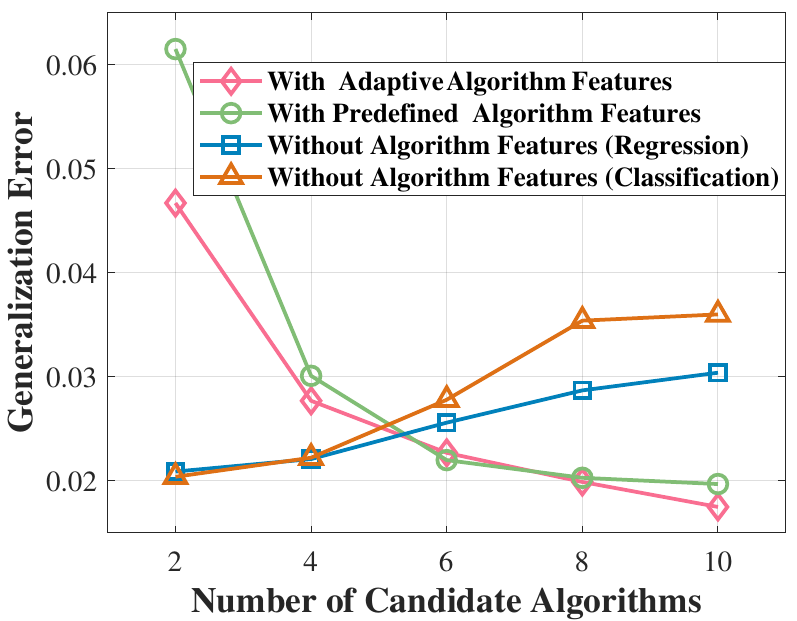}}
\subfigure[Generalization Error with 20,000 Problems]{ \centering
    \label{Problem2WGE}
    \includegraphics[height=1.5in,width=1.65in]{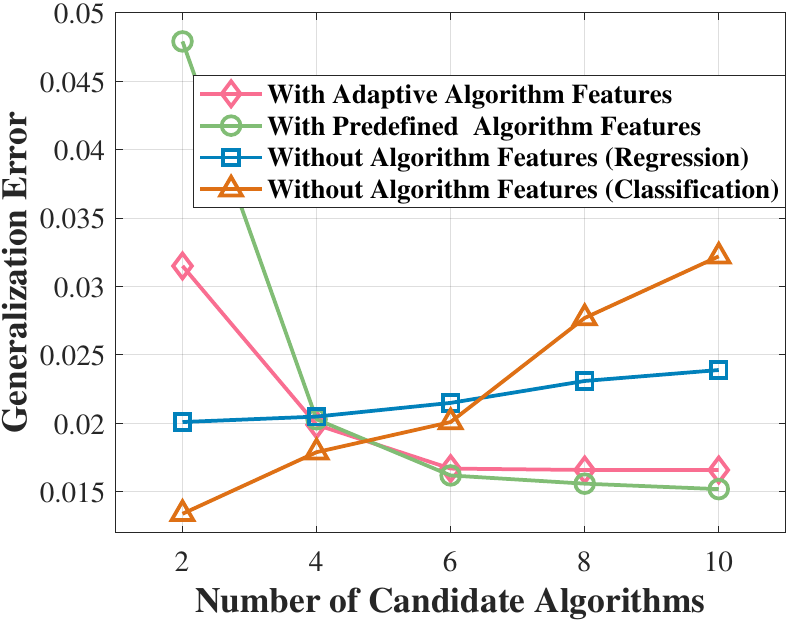}}
\caption{The impact of the number of candidate algorithms on model performance.}
\label{AlgorithmScale}
\end{figure}

\subsection{Impact of the Distribution Shift}
\label{ex4}

\begin{figure}[t]
\centering%
\subfigure[Under Algorithm Distribution Shift]{ \centering
    \label{DistributionShift1}
    \includegraphics[height=1.9in,width=3.4in]{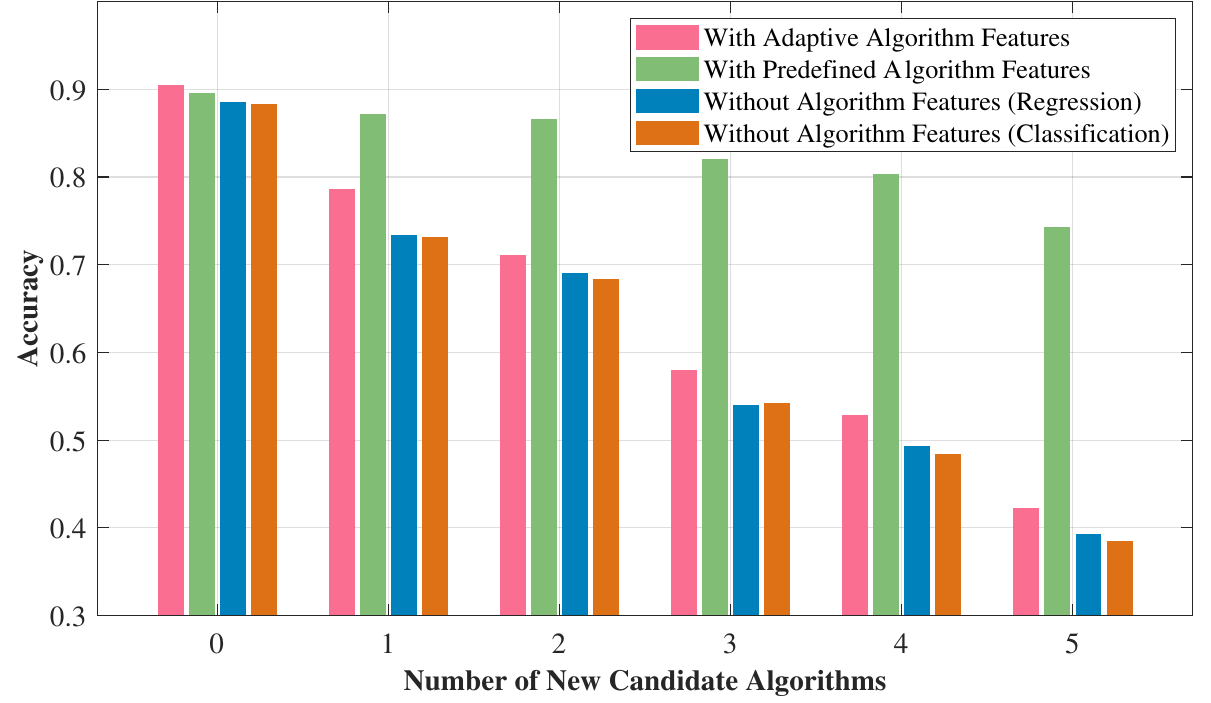}}
\subfigure[Under Problem Distribution Shift]{ \centering
    \label{DistributionShift2}
    \includegraphics[height=1.9in,width=3.4in]{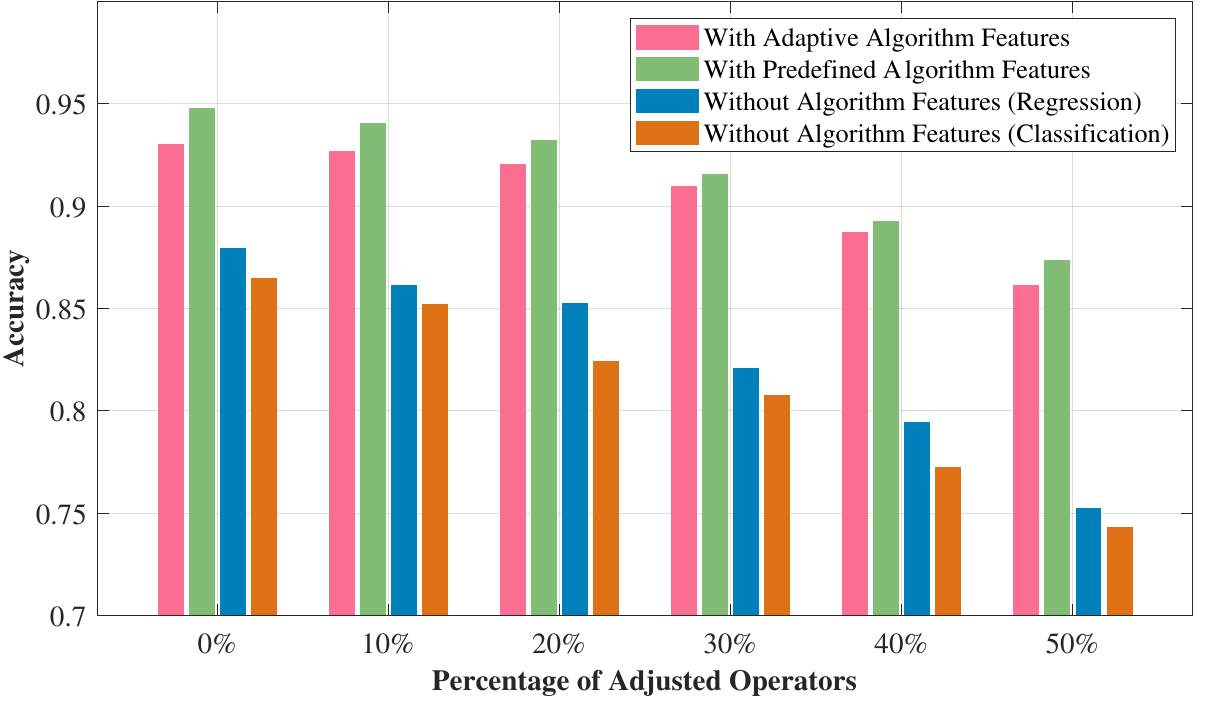}}
\caption{The impact of the distribution shift on model performance.}
\label{DistributionShift}
\end{figure}

Based on Corollary \ref{theorem5}, predefined algorithm feature-based models demonstrate generalization capabilities in the presence of distribution shift between training and testing data, a benefit not shared by models relying on adaptive features or solely problem features. To showcase the robustness of predefined algorithm features under distribution shifts, we manipulate the distributions of both algorithms and problems and assess the performance variations across different methods. This distributional adjustment is achieved by controlling the sets of candidate algorithms and operator usage within the problem set. We compare four methods, including two utilizing algorithm features (adaptive features and predefined features) and two not utilizing algorithm features (regression and classification modeling approaches). In each experiment, we maintain a fixed number of 10,000 training problems and 2,500 testing problems.

The first experiment investigates the influence of algorithm distribution shifts on model performance. We begin by fixing the number of candidate algorithms in the training set at 5 and gradually introduce 0-5 new candidate algorithms into the testing set. The performance differences among the models are visualized in Fig. \ref{DistributionShift1}. Across all experimental groups with distribution shifts (excluding the first group), models constructed using predefined algorithm features consistently exhibit an advantage, with this advantage increasing as the distribution shifts become more pronounced. In contrast, models based on adaptive features, which represent each algorithm through an embedding layer, lack generalization on the algorithmic level, similar to models relying solely on problem features. As new algorithms are introduced in the testing set, these models fail to consider any information related to the new algorithms, resulting in significant performance degradation as the number of new algorithms increases. On the other hand, models leveraging predefined algorithm features can capture patterns within the algorithm set, resulting in smaller performance losses and demonstrating the algorithmic generalization of predefined algorithm features.

The second experiment delves into the impact of problem distribution shifts on model performance. We initially fix the probability of operator usage in the testing set and gradually adjust the usage probabilities of a certain proportion of operators in the training set. By reducing the usage probability of $10\%-50\%$ of operators to $10\%$ of their initial value, we manipulate the problem distribution. The performance differences among the models are illustrated in Fig. \ref{DistributionShift2}. Overall, as more operators are adjusted, each method exhibits slight performance declines. However, the influence of problem distribution shifts on model performance is weaker compared to that of algorithm distribution shifts. Models utilizing algorithm features experience smaller performance declines compared to models relying solely on problem features. This effect can be attributed to the consideration of algorithm features, which enables the models to capture the interrelationship between algorithms and problems, resulting in more robust models with a bidirectional mapping from problems to algorithms. The overall results presented in Fig. \ref{DistributionShift} align with the findings of Corollary \ref{theorem5}, highlighting the impact of distributional shifts on model generalization under predefined algorithm features.

\subsection{Impact of the Training Scale under Distribution Shift}
\label{ex5}

\begin{figure}[t]
\centering%
\subfigure[Accuracy]{ \centering
    \label{DistributionShift3}
    \includegraphics[height=1.5in,width=1.65in]{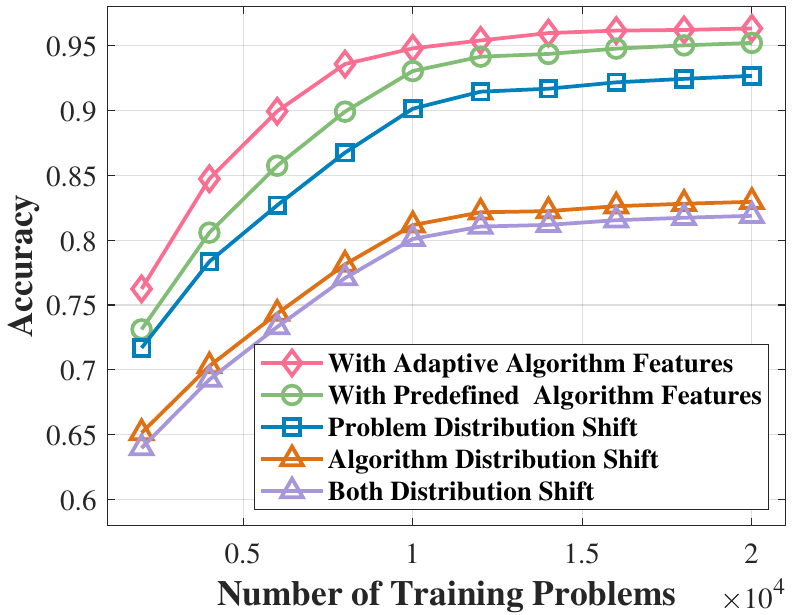}}
\subfigure[Generalization Error]{ \centering
    \label{DistributionShift3GE}
    \includegraphics[height=1.5in,width=1.65in]{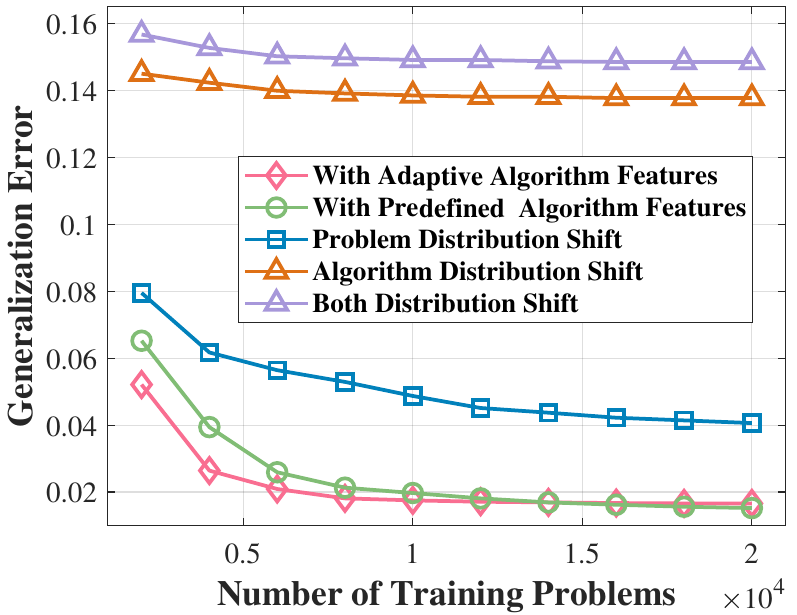}}
\caption{The impact of the number of problem instances on model performance under distribution shift.}
\label{ProblemScale_DistributionShift}
\end{figure}

Comparing Theorem \ref{bound2} and Theorem \ref{theorem4}, we observe that the influence of training sample size on generalization error differs between transductive learning and inductive learning. This discrepancy primarily stems from the fact that the performance-enhancing effect of increasing the training data is attenuated by distributional differences when accounting for distribution shifts. To demonstrate the mediating effect of distribution shifts on performance improvement during the augmentation of training data, this subsection investigates performance variations under different distributional shifts. We progressively augment the number of problem instances used for training and compare five distinct scenarios, including two methods utilizing algorithm features (adaptive features and predefined features) under no distribution shift, the performance of the method utilizing predefined algorithm features under problem distribution shift, algorithm distribution shift, and both distribution shifts simultaneously, with the implementation of distribution shifts following the same procedure as in Section 5.4. Fig. \ref{ProblemScale_DistributionShift} illustrates the performance trends across different scenarios. It is evident that as the number of problem instances in the training set increases, all models show a growth trend in performance, though with varying magnitudes. Notably, the impact of algorithmic distribution shift on performance is more pronounced. According to Theorem \ref{bound2}, in the absence of distribution shifts, the upper bound of generalization error in transductive learning increases at a rate of approximately $\frac{1}{|\mathcal{S}_{\mathcal{P}}|^{\frac{1}{2}}}$ (reaching $\frac{1}{|\mathcal{S}_{\mathcal{P}}|}$ when the number of samples is small), while in inductive settings, the rate is approximately $\frac{1}{|\mathcal{S}_{\mathcal{P}}|^{\frac{1}{4}}}$ according to Theorem \ref{theorem4}. The slope differences depicted in Figure \ref{ProblemScale_DistributionShift} align with the aforementioned theoretical conclusions, as models exhibit more rapid performance improvement in the absence of distributional shifts, whereas the presence of distribution shifts, particularly changes in algorithmic distribution, leads to flatter performance growth curves.

\subsection{Impact of the Model Complexity}
\label{ex6}

\begin{figure}[t]
\centering%
\subfigure[Accuracy on the test set.]{ \centering
    \label{ModelComplexity1}
    \includegraphics[height=1.5in,width=1.65in]{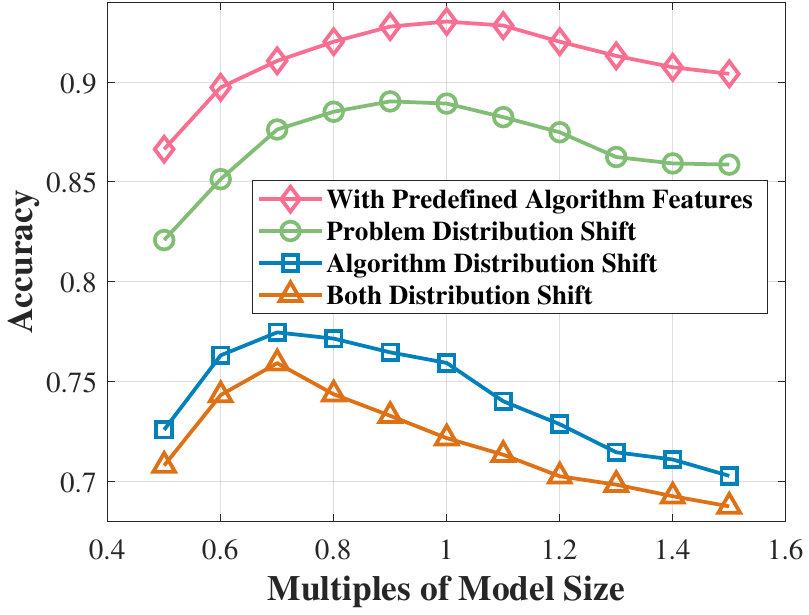}}
\subfigure[Error in accuracy between the test set and training set]{ \centering
    \label{ModelComplexity2}
    \includegraphics[height=1.5in,width=1.65in]{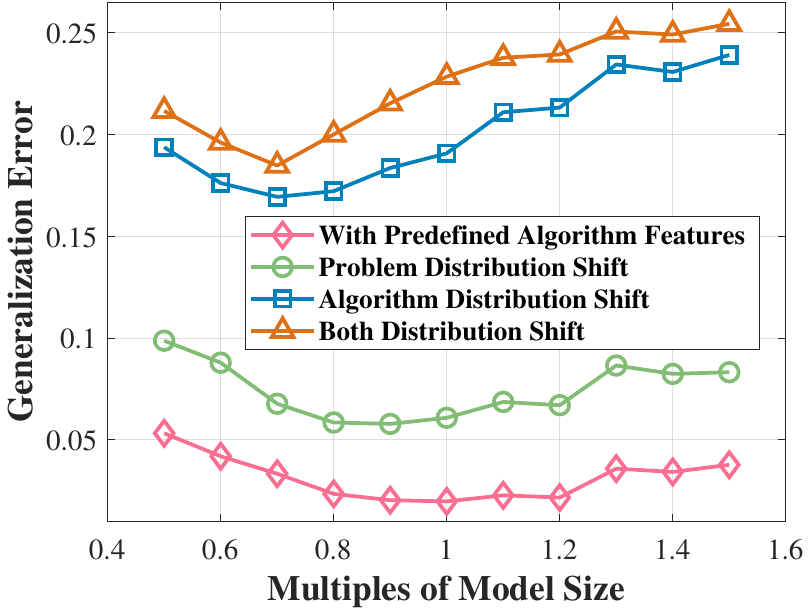}}
\caption{The impact of the model complexity on model performance.}
\label{ModelComplexity}
\end{figure}

Based on Corollary \ref{theorem5}, both $\chi^2(P_{\mathcal{T}} \| P_{\mathcal{S}})$ and $\prod_{i=1}^l R_i$ jointly affect the upper bound of the model's generalization error. The latter is typically associated with the complexity of the model. Therefore, we recommend avoiding excessively complex models in scenarios with significant distributional shifts. To validate this proposition, this section conducts experiments to examine the relationship between model complexity and performance. We fix the number of candidate algorithms and problem instances in the training set at 10 and 10,000, respectively. Following the reference model used in Section 5.2, we adjust the number of neurons in all hidden layers of the multilayer perceptron to be $k$ times that of the reference model ($k\in\{0.5,0.6,\cdots,1.5\}$), thereby obtaining models with varying degrees of complexity. The adjusted models are trained on the same training data until convergence, and the optimal model performance and error results are recorded. We compare the differences in the relationship between model complexity and performance across scenarios with/without distribution shifts. Fig. \ref{ModelComplexity1} and Fig. \ref{ModelComplexity2} illustrate the changes in model accuracy on the test set and the error in accuracy between the test set and training set, respectively. From the figures, it is evident that when distribution shifts are present, the model complexity corresponding to the optimal performance is lower than that in scenarios without distributional shifts. Specifically, when only the problem distribution shift exists, the model size corresponding to the optimal accuracy/minimum error is $90\%$ of the size of the reference model. In the presence of an algorithmic distribution shift, regardless of whether a problem distribution shift exists or not, the model size corresponding to the optimal accuracy/minimum error is $70\%$ of the size of the reference model. The reason why models that are even simpler than these optimal complexities fail to achieve better performance is primarily because overly simple models cannot fully capture the underlying mechanisms of algorithm selection. This indicates that model complexity amplifies the generalization error caused by distribution shifts. When the test data's distribution deviates significantly from that of the training data, it is advisable to design models of an appropriate size based on the scale of the training set. Excessively complex models are not conducive to generalizing to test scenarios with substantial distributional differences. These findings align with the theoretical results in Corollary \ref{theorem5} of this paper.

\section{Conclusion}
\label{conclusion}

This paper delves into the complex interplay between algorithm features and problem features in the realm of algorithm selection, offering both theoretical and empirical insights. Our comprehensive study reveals several key findings and provides actionable guidelines for the practical application of algorithm features in real-world scenarios.

Our theoretical analysis highlights the distinct impacts of training set size on the Rademacher complexity under both transductive and inductive learning paradigms. We found that transductive learning models benefit more from larger training sets in terms of reducing Rademacher complexity compared to inductive learning models. This insight underscores the importance of tailoring training strategies to the learning paradigm employed. In the context of transductive learning, our results indicate that algorithm features improve model generalization, especially when dealing with a large number of candidate algorithms. The generalization error bounds for models leveraging algorithm features decrease more rapidly with increasing training data size, compared to models based solely on problem features. This advantage becomes more pronounced as the amount of training data grows, suggesting that algorithm feature-based models are particularly suited for scenarios with extensive training problems and a large number of candidate algorithms. When considering distribution shifts between the training and test sets, models based on algorithm features generally exhibit better generalization than those relying only on problem features. However, predefined algorithm features demonstrate superior generalization performance compared to adaptive features. While adaptive features adapt to the training data, they lack the robustness to handle new candidate algorithms in different distributions effectively. On the other hand, predefined features, which capture the intrinsic properties of algorithms, provide a more consistent generalization capability even under significant distribution shifts. Nonetheless, even predefined features suffer an increase in generalization error proportional to the chi-square distance between the training and test distributions. This finding emphasizes the need for careful consideration of distribution shifts in practical applications and suggests a cautious approach to using large-scale models in such scenarios.

In summary, our research provides a nuanced understanding of when and how to utilize algorithm features effectively. We recommend leveraging algorithm features in settings with a large number of candidate algorithms and considering the use of predefined features to handle potential distribution shifts. For situations involving significant distribution shifts, predefined features offer a more robust option, though the model complexity should be managed to mitigate potential increases in generalization error. By following these guidelines, practitioners can better harness the power of algorithm features to improve the accuracy and reliability of algorithm selection processes.

\ifCLASSOPTIONcaptionsoff
  \newpage
\fi

\bibliographystyle{IEEEtran}
\bibliography{bare_jrnl}



\end{document}